\documentclass{article}

\usepackage{yk}

    \usepackage[final,nonatbib]{neurips_2021}
    \usepackage[numbers]{natbib}

\usepackage[utf8]{inputenc} 
\usepackage[T1]{fontenc}    
\usepackage{hyperref}       
\usepackage{url}            
\usepackage{booktabs}       
\usepackage{amsfonts}       
\usepackage{nicefrac}       
\usepackage{microtype}      
\usepackage{xcolor}         

\usepackage{multirow}

\usepackage{amsmath,amsfonts,bm}

\def\eqref#1{equation~\ref{#1}}

\def\1{\bm{1}}

\DeclareMathAlphabet{\mathsfit}{\encodingdefault}{\sfdefault}{m}{sl}
\SetMathAlphabet{\mathsfit}{bold}{\encodingdefault}{\sfdefault}{bx}{n}

\DeclareMathOperator*{\argmin}{arg\,min}

\usepackage{subfigure}
\usepackage{graphicx}
\usepackage{algorithmic}
\usepackage[linesnumbered,ruled]{algorithm2e}
\usepackage{algorithm}
\usepackage{wrapfig}
\usepackage{enumitem}

\title{On sensitivity of meta-learning to support data}

\author{
    Mayank Agarwal $^{1}$ \\
    \texttt{mayank.agarwal@ibm.com} \\
    \And
    Mikhail Yurochkin $^{1,2}$   \\
    \texttt{mikhail.yurochkin@ibm.com}  \\
    \And
    Yuekai Sun  $^{3}$        \\
    \texttt{yuekai@umich.edu}   \\
    \And
    \normalfont{IBM Research,$^1$ MIT-IBM Watson AI Lab,$^2$ University of Michigan$^3$.}
}

\begin{document}

\maketitle

\begin{abstract}
  Meta-learning algorithms are widely used for few-shot learning. For example, image recognition systems that 
  readily adapt to unseen classes after seeing only a few labeled examples. Despite their success, we show that 
  modern meta-learning algorithms are extremely sensitive to the data used for adaptation, i.e. support data. In particular, we 
  demonstrate the existence of (unaltered, in-distribution, natural) images that, when used for adaptation, yield 
  accuracy as low as 4\% or as high as 95\% on standard few-shot image classification benchmarks. We explain our 
  empirical findings in terms of class margins, which in turn suggests that robust and safe meta-learning requires 
  larger margins than supervised learning.
\end{abstract}


\section{Introduction}
\label{sec:intro}
Meta-learning, or learning to learn \citep{novak1984learning}, is the problem of training models that can adapt to new tasks quickly, using only a handful of examples. The problem is inspired by humans' ability to learn new skills or concepts at a rapid pace (\eg\ recognizing previously unknown objects after seeing only a single example \citep{lake2011one}). Meta-learning has found applications in many domains, including safety-critical medical image analysis \citep{maicas2018training}, autonomous driving \citep{sallab2017meta}, visual navigation \citep{wortsman2019learning} and legged robots control \citep{song2020rapidly}. In this paper, we investigate the vulnerabilities of modern meta-learning algorithms in the context of few-shot image classification problem \citep{miller2000learning}, where a meta-learner needs to solve a classification task on classes unseen during training using a small number (typically 1 or 5) of labeled samples from each of these classes to adapt. Specifically, we demonstrate that the performance of modern meta-learning algorithms on few-shot image recognition benchmarks varies drastically depending on the examples provided for adaptation, typically called \emph{support data}, raising concerns about its safety in deployment.

Despite the many empirical successes of artificial intelligence, its vulnerabilities are important to explore and mitigate in our pursuit of safe AI that is suitable for critical applications. Sensitivity to small (possibly adversarial) perturbations to the inputs \citep{goodfellow2014explaining}, backdoor attacks allowing malicious model manipulations \citep{chen2017targeted}, algorithmic biases \citep{angwin2016Machine} and poor generalization outside of the training domain \citep{koh2020wilds} are some of the prominent examples of AI safety failures. Some of these issues have also been studied in the context of meta-learning, e.g. adversarial robustness \citep{yin2018adversarial,goldblum2019adversarially,xu2020yet} and fairness \citep{slack2020fairness}. Most of the aforementioned AI-safety dangers are associated with adversarial manipulations of the train or test data, or significant distribution shifts at test time. In meta-learning, 
prior works demonstrated that an adversary can create visually imperceptible changes of the test inputs \citep{goldblum2019adversarially} or the support data \citep{xu2020yet, oldewage2020attacking} causing meta-learning algorithms to fail on various few-shot image recognition benchmarks. In this work we demonstrate \emph{adversary-free} and \emph{in-distribution} failures specific to meta-learning. \citet{sohn2020fixmatch} studied a similar problem in the context of semi-supervised learning.

\begin{figure*}
    \centering
    \includegraphics[width=\textwidth]{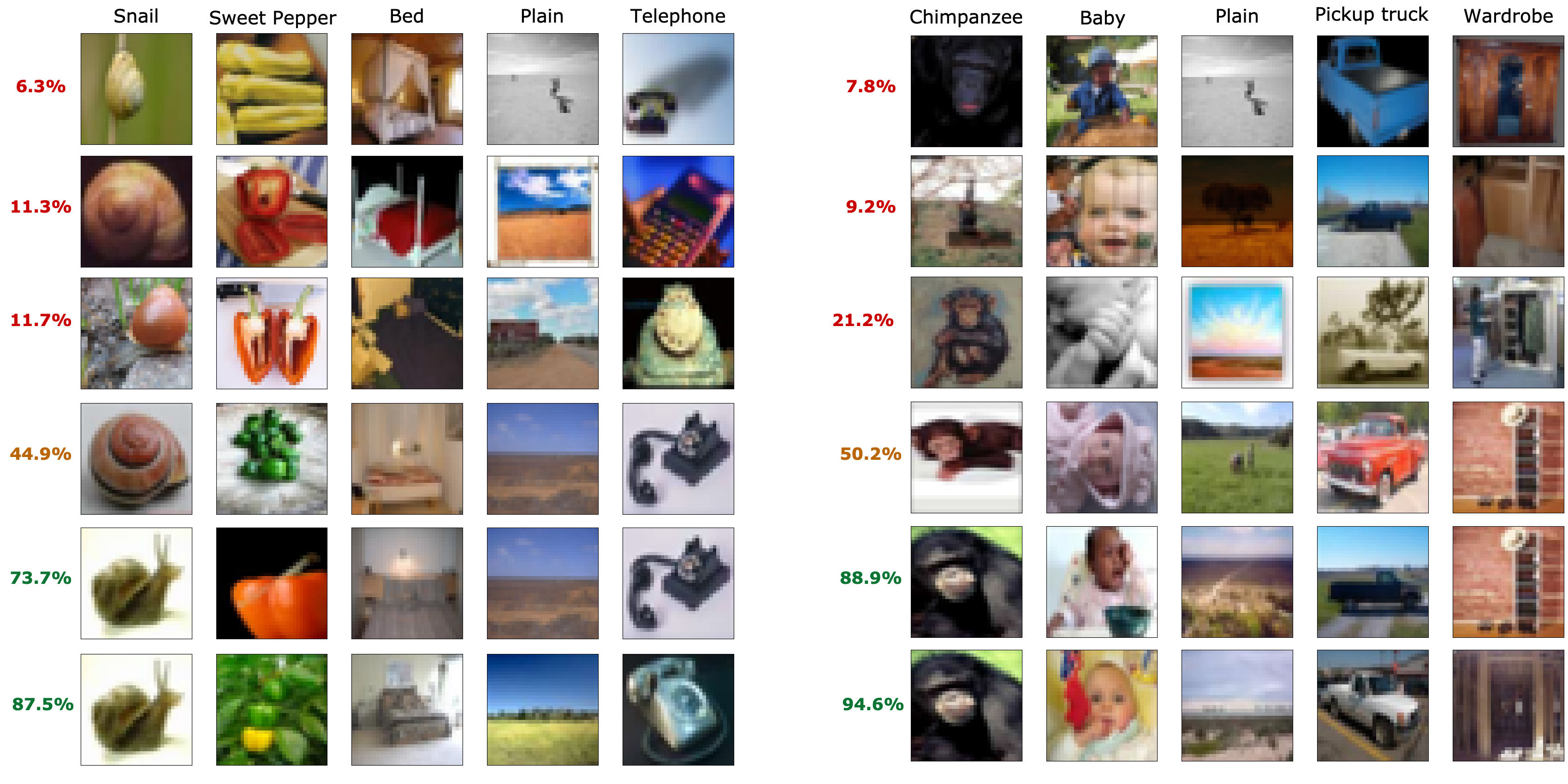}
    \caption{Examples of CIFAR-FS support images (unaltered, i.e. not modified adversarially or in any other way) from two test tasks that yield vastly different 1-shot post-adaptation performances of a popular meta-learning algorithm (MetaOptNet-SVM). Accuracy ranges from 6.3\% to 94.6\% suggesting extreme sensitivity of the meta-learner to the support data. Images mostly appear representative of the corresponding classes and would be hard to recognize as potentially problematic without significant expertise of the training dataset from the human labeling the support data.}
    \label{fig:exemplars}
\end{figure*}



A distinct feature of meta-learning is the adaptation stage where the model is updated using a scarce amount of labeled examples from a new task. In practice, these examples need to be selected and presented to a human for labeling. Then a meta-learner adapts and can be used to classify the remaining images without a human. Currently, meta-learning algorithms choose support examples for labeling at random. \citet{dhillon2019baseline} noted that the standard deviation of accuracies computed with random support examples may be large. We demonstrate that if the support data is not carefully curated, the performance of the meta-learner can be unreliable. In Figure \ref{fig:exemplars} we present examples of unaltered support images (i.e. they are not modified adversarially or in any other way) representative of the corresponding task (in-distribution) that lead to vastly different performance of the meta-learner after adaptation. Support examples causing poor performance are not prevalent, however, they are also not data artifacts as there are multiple of them. The existence of such examples might be concerning when deploying meta-learning in practice, even when adversarial intervention is out of scope.

Our main contributions are as follows:
\begin{enumerate}[topsep=0pt,itemsep=1ex,partopsep=1ex,parsep=1ex]
\item We present a simple algorithm for finding the best and the worst support examples and empirically demonstrate the sensitivity of popular meta-learning algorithms to support data.
\item We demonstrate that a popular strategy for achieving robustness \citep{madry2017towards} adapted to our setting \emph{fails} to solve the support data sensitivity problem.
\item We explain the existence of the worst case examples from the margin perspective suggesting that robust meta-learning requires class margins more stringent than classical supervised learning.
\end{enumerate}

\section{Meta-learning approaches}
\label{sec:maml}
Meta-learning approaches are typically categorized into model-based \citep{ravi2016optimization,mishra2017simple,qiao2018few}, metric-based \citep{koch2015siamese,vinyals2016matching,snell2017prototypical} and optimization-based methods \citep{finn2017model,nichol2018first,grefenstette2019generalized}. We refer to \citet{hospedales2020meta} for a recent survey of the area. In this paper we mostly focus on the optimization-based method popularized by the Model Agnostic Meta Learning (MAML) \citep{finn2017model} bi-level formulation (metric-based prototypical networks \citep{snell2017prototypical} are also considered in the experiments). Let $\{\mathcal{T}_n\}_{i=1}^n$ be a collection of $n$ tasks $\mathcal{T}_i = \{\mathcal{A}_i, \mathcal{D}_i\}$ each consisting of a support $\mathcal{A}_i$ and a query (or validation) $\mathcal{D}_i$ datasets. MAML bi-level optimization problem is as follows \citep{finn2017model}:
\begin{equation}
\label{eq:maml}
\begin{split}
\min_\theta & \frac{1}{n}\sum_{i=1}^n \ell(\theta_i, \theta; \mathcal{D}_i),\\
\text{such that } & \,\theta_i = \argmin_{\theta_i} \ell(\theta_i,\theta; \cA_i),\,i=1,\dots,n.
\end{split}
\end{equation}
The first line of \eqref{eq:maml} is typically minimized using gradient-based optimization of $\theta$ and is called \emph{meta-update}. The second line is the \emph{adaptation} step and its implementation differs depending on whether the fine-tuned $\theta_i$ are treated as parameters of the end-to-end neural network or as the ``head'' parameters, i.e. the last linear classification layer. We refer to \citet{goldblum2020unraveling} for an empirical study of the differences between the two adaptation perspectives.

MAML \citep{finn2017model} and Reptile \citep{nichol2018first} are examples of algorithms that fine-tune all network parameters during the adaptation. Instead of solving the argmin of $\theta_i$ exactly they approximate it with a small number of gradient steps. For one gradient step approximation \eqref{eq:maml} can be written as
\begin{equation}
\label{eq:maml-step}
\min_\theta \frac{1}{n}\sum_{i=1}^n \ell(\theta - \alpha\frac{\partial \ell(\theta;\cA_i)}{\partial \theta};\cD_i)
\end{equation}
for some step size $\alpha$. MAML directly optimizes \eqref{eq:maml-step} differentiating through the inner gradient step, which requires expensive Hessian computations. Reptile introduces an approximation by-passing the Hessian computations and often performs better.

Another family of meta-learning algorithms considers $\theta$ as the neural network feature extractor parameters shared across tasks and adapts only the linear classifier parametrized with $\theta_i$, that takes the features extracted with $\theta$ as inputs. The advantage of this perspective is that the adaptation minimization problem is convex and, for many linear classifiers, can be solved fast and exactly. Let $a(\theta,\cA)$ denote a procedure that takes data in $\cA$, passes it through a feature extractor parametrized with $\theta$, and returns $\theta_i$, i.e. optimal parameters of a linear classifier using the obtained features to predict the corresponding labels. Then \eqref{eq:maml} can be written as
\begin{equation}
\label{eq:maml-head}
\min\limits_\theta \frac{1}{n}\sum_{i=1}^n \ell(\theta,a(\theta,\cA_i);\cD_i).
\end{equation}
This approach requires $a(\theta,\cA)$ to be differentiable. R2D2 \citep{bertinetto2019meta} casts classification as a multi-target ridge-regression problem and utilizes the corresponding closed-form solution as $a(\theta,\cA)$. \citet{lee2019meta} propose MetaOptNet, where $a(\theta,\cA)$ is a differentiable quadratic programming solver \citep{amos2017optnet}, and implement it with linear support vector machines (SVM) and ridge-regression. Prototypical networks \citep{snell2017prototypical} is a metric-based approach that can also be viewed from the perspective of \eqref{eq:maml-head}. Here $a(\theta,\cA)$ outputs class centroids in the neural-feature space and the predictions are based on the closest class centroids. Last-layer adaptation approaches, especially MetaOptNet, outperform full-network approaches on most benchmarks \citep{lee2019meta,goldblum2020unraveling}.

Both adaptation perspectives have a weakness underlying our study: their test performance is based on an optimization problem with large number of parameters and as little as 5 data points (1-shot 5-way setting). To understand the problem, consider the last-layer adaptation approach: even when the feature extractor produces linearly separable representations, the dimension is large (for example, \citet{lee2019meta} utilize a ResNet-12 architecture with 2560-dimensional last layer features that we adopt in our experiments) making the corresponding linear classifier extremely sensitive to the support data. In Section \ref{sec:adaptation} we demonstrate the problem empirically and provide theoretical insights in Section \ref{sec:margins}.

\subsection{Meta-learning benchmarks}
\label{sec:data}
Before presenting our findings, we discuss the meta-learning benchmarks we consider. Meta-learning algorithms are often compared in a few-shot image recognition setting. Each task typically has five unique classes, i.e. 5-way, and one or five examples per class for adaptation, i.e. 1-shot or 5-shot. Classes that appear at test time are not seen during training. Few-shot learning datasets are typically the derivatives of the existing supervised learning benchmarks, e.g. CIFAR-100 \citep{krizhevsky2009learning} and ImageNet \citep{deng2009imagenet}. Few-shot problem is setup by disjoint partitioning of the available classes into train, validation and test. Tasks are obtained by sampling 5 distinct classes from the corresponding pool and assigning a subset of examples for adaptation and a different subset for meta-update during training or accuracy evaluation for reporting the performance.

CIFAR-FS \citep{bertinetto2019meta} is a dataset of 60000 32$\times$32 RGB images from CIFAR-100 partitioned into 64, 16 and 20 classes for training, validation and testing, respectively. FC-100 \citep{oreshkin2018tadam} is also a derivative of CIFAR-100 with a different partition aimed to reduce semantic overlap between 60 classes assigned for training, 20 for validation, and 20 for testing. MiniImageNet \citep{vinyals2016matching} is a subsampled, downsized version of ImageNet. It consists of 60000 84$\times$84 RGB images from 100 classes split into 64 for training, 16 for validation, and 20 for testing.


\section{Finding adaptation vulnerabilities}
\label{sec:adaptation}
In this section we study the performance range of meta-learners trained with a variety of algorithms: MAML \citep{finn2017model}, Meta-Curvature (MC) \citep{NEURIPS2019_57c0531e}, Prototypical networks \citep{snell2017prototypical}, R2D2 \citep{bertinetto2019meta}, and MetaOptNet with Ridge and SVM heads \citep{lee2019meta}. At test time a meta-learner receives support data and its performance is recorded after it adapts. Typical evaluation protocols sample support data randomly and report average performance across tasks. While this is a useful measure for comparing algorithms, safety-critical applications demand understanding of potential performance ranges. We demonstrate that there are support examples the lead to vastly different test performances. To identify the worst and the best case support examples for a given task we perform an iterative greedy search.

Let $\cX = \{x^k_m\}^{k\in[K]}_{m\in[M]}$ be a set of potential support examples for a given task. $K$ is the number of classes, $M$ is the number of examples per class (we assume same $M$ for each class for simplicity), and $[M]=\{0,\dots,M-1\}$. Let $\cD$ be the evaluation set for this task. In a $J$-shot setting, let $\cZ = \{z^k_j\}^{k\in[K]}_{j\in[J]}$ be a set of \emph{indices} of the support examples ($z^k_j \neq z^k_{j'}$ for any $k,j\neq j'$ and $z^k_j \in [M]$ for any $k,j$). Denote $R(\cZ,\cX,\cD)$ a function computing accuracy of a meta learner on the query set $\cD$ after adapting on $\cA=\{x^k_{z^k_j}\}^{k\in[K]}_{j\in[J]}$. Finding the worst/best case accuracy amounts to finding indices $\cZ$ minimizing/maximizing $R(\cZ,\cX,\cD)$. We solve this greedily by updating a single index $z^k_j$ at a time, holding the rest of $\cZ$ fixed, iterating over $[K]$ and $[J]$ multiple times. We summarize the procedure for finding the worst case accuracy in Algorithm \ref{alg:attack} (the best case accuracy is analogous).\footnote{This is a simple approach, however we found it faster and more efficient then a more sophisticated continuous relaxation using weights of the potential support examples with sparsity constraints.}
Due to the greedy nature, this algorithm finds a local optima, but it is sufficiently effective as we see in the following section.

\begin{algorithm}
   \caption{Finding the worst case support examples}
   \label{alg:attack}
\begin{algorithmic}
\STATE {\bfseries Input:} trained meta-learner, potential support examples $\cX$ to search over, query data $\cD$.
   \REPEAT
   \STATE Initialize indices $\cZ=\{z^k_j\}^{k\in[K]}_{j\in[J]}$ randomly.
   \FOR{$j \in [J]$, $k \in [K]$}
   \STATE $z^k_j \leftarrow \argmin_{z^k_j} R(\cZ,\cX,\cD)$, $z^k_j \in [M]\setminus \{z^k_{j'}\}_{j'\neq j}$
   \ENDFOR
   \UNTIL{$R(\cZ,\cX,\cD)$ stops decreasing}
 \STATE {\bfseries Output:} support examples indices $\cZ = \{z^k_j\}^{k\in[K]}_{j\in[J]}$
\end{algorithmic}
\end{algorithm}

\paragraph{Visualizing the worst case support search} In Figure \ref{fig:algorithm-illustration} we visualize a single iteration of Algorithm \ref{alg:attack} on a 1-shot task from the CIFAR-FS dataset. This is a 5-way task with ``snail'', ``red pepper'', ``bed'', ``plain'', and ``telephone'' as classes. In round 1, i.e. $k=0$, of iteration 1 we start with a random example per class and evaluate post-adaptation accuracy on the query data $\cD$ for each potential support ``snail'', i.e. $\{x^0_m\}_{m \in [M]}$. Here the first line corresponds to ``snail'' indexed with $m=0$, i.e. $x^0_0$, and post-adaptation accuracy of 73.5\%.
We select a ``snail'' support image corresponding to the worst accuracy of 53.8\% and proceed to round 2, i.e. $k=1$, of iteration 1. In round 2 we repeat the same procedure for the corresponding class ``red pepper'' (again the first line corresponds to ``red pepper'' indexed with $m=0$, i.e. $x^1_0$), holding the support ``snail'' image selected previously and support images for other classes fixed. The algorithm finds the worst case support ``red pepper'' image corresponding to the post-adaptation accuracy of 37\%. Then the algorithm proceeds analogously to rounds 3, 4, and 5 of iteration 1, resulting in a combination of support examples corresponding to 2.9\% post-adaptation accuracy. This is already sufficiently low, however on some tasks and higher-shot settings it may be beneficial to run additional iterations. On iteration 2, the algorithm will again go through all the classes starting from the support examples found on iteration 1 instead of random ones. In our experiments we always run Algorithm \ref{alg:attack} for 3 iterations. In Appendix \ref{sec:appendix:convergence} we empirically study the convergence of the algorithm justifying this choice.

\begin{figure}
    \centering
    \includegraphics[width=\textwidth]{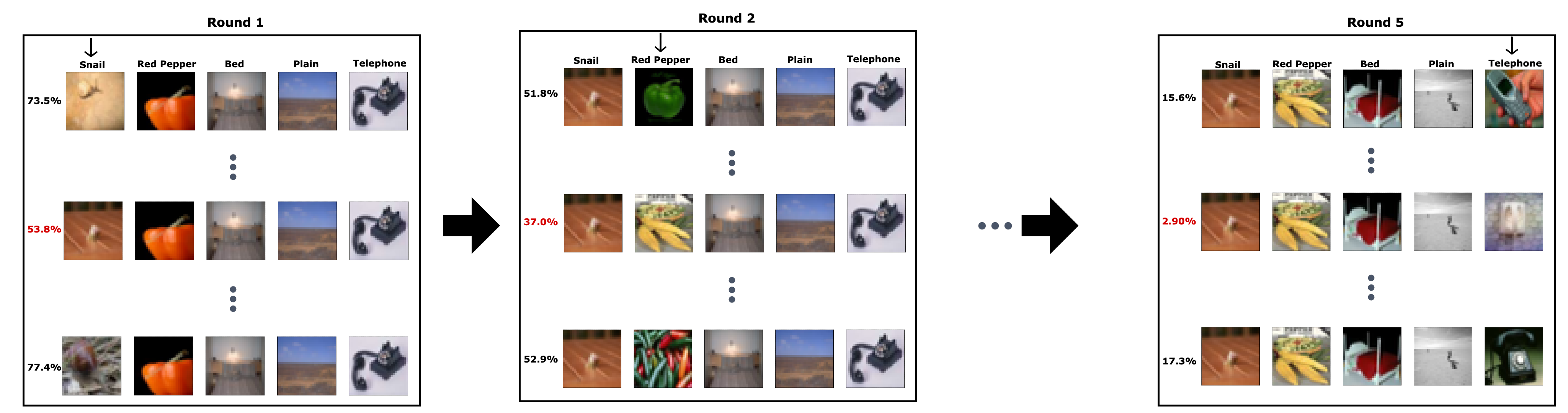}
    \caption{Visualization of Algorithm \ref{alg:attack} on a random 1-shot task for the MetaOptNet-SVM method on the CIFAR-FS dataset. We show 1 iteration of the algorithm here, where the algorithm starts by randomly sampling a support image per class and then iterates over all data samples for all classes.}
    \label{fig:algorithm-illustration}
\end{figure}

\begin{table}[h]
    \caption{Accuracies for different meta-learning methods on the CIFAR-FS dataset}
    \label{tab:accs:maml:cifar_fs}
    \centering
    \begin{small}
    \begin{tabular}{ccccc}
    \toprule
              &           &       Worst acc       &       Avg acc         &       Best acc        \\
    \toprule
    
    
    
    \multirow{3}{*}{MAML}   &   1-shot      &       5.91 $\pm$ 1.94\%        &      56.52 $\pm$ 10.85\%       &       80.04 $\pm$ 6.83\%      \\
                            &   5-shot      &       13.88 $\pm$ 7.28\%        &      70.45 $\pm$ 8.42\%       &       85.15 $\pm$ 6.12\%     \\
                            &   10-shot      &       26.26 $\pm$ 7.85\%        &      70.88 $\pm$ 7.82\%       &       85.81 $\pm$ 6.47\%     \\
                                    \cmidrule(r){1-5}
                                    
    \multirow{3}{*}{MC}   &   1-shot      &       4.73 $\pm$ 1.51\%        &      47.76 $\pm$ 11.28\%       &       69.81 $\pm$ 8.37\%      \\
                            &   5-shot      &       9.29 $\pm$ 6.13\%        &      70.23 $\pm$ 8.82\%       &       85.22 $\pm$ 6.53\%      \\
                            &   10-shot      &       16.95 $\pm$ 10.78\%        &      70.62 $\pm$ 7.89\%       &       87.81 $\pm$ 6.91\%     \\
                                    \cmidrule(r){1-5}
                                    
    \multirow{3}{*}{ProtoNets}   &   1-shot      &       5.06 $\pm$ 2.37\%        &      63.80 $\pm$ 0.71\%       &       84.88 $\pm$ 6.29\%        \\
                                 &   5-shot      &    18.28 $\pm$ 9.85\%         &      80.06 $\pm$ 0.46\%       &       90.41 $\pm$ 4.72\%      \\
                                  &   10-shot      &   27.32 $\pm$ 11.17\%        &      82.95 $\pm$ 0.44\%       &       91.44 $\pm$ 4.35\%     \\
                                    \cmidrule(r){1-5}
                                    
    \multirow{3}{*}{R2D2}       &   1-shot      &       6.14 $\pm$ 2.89\%        &      68.86 $\pm$ 0.68\%       &       86.63 $\pm$ 5.97\%    \\
                                  &   5-shot      &       14.73 $\pm$ 8.21\%        &      82.29 $\pm$ 0.44\%       &       92.34 $\pm$ 4.21\%       \\
                               &   10-shot      &       29.30 $\pm$ 12.16\%        &      85.74 $\pm$ 0.42\%       &       93.41 $\pm$ 3.81\%       \\
                                    \cmidrule(r){1-5}

    \multirow{3}{*}{MetaOptNet-Ridge}   &   1-shot      &       5.17 $\pm$ 2.97\%        &      71.21 $\pm$ 0.67\%       &       87.40 $\pm$ 5.95\%    \\
                                       &   5-shot      &       16.81 $\pm$ 11.57\%        &      84.18 $\pm$ 0.45\%       &       93.15 $\pm$ 4.43\%        \\
                                      &   10-shot      &       32.10 $\pm$ 16.26\%        &      86.82 $\pm$ 0.42\%       &       93.89 $\pm$ 3.59\%       \\
                                    \cmidrule(r){1-5}

    \multirow{3}{*}{MetaOptNet-SVM}   &   1-shot      &       5.27 $\pm$ 2.82\%        &      70.79 $\pm$ 0.69\%       &       87.65 $\pm$ 5.76\%     \\
                                       &   5-shot      &       14.92 $\pm$ 8.60\%        &      83.98 $\pm$ 0.44\%       &       93.36 $\pm$ 4.60\%        \\
                                       &   10-shot      &       22.24 $\pm$ 10.13\%        &      87.11 $\pm$ 0.40\%       &       93.56 $\pm$ 3.98\%      \\
    \bottomrule
    
    \end{tabular}
    \end{small}
\end{table}

\begin{table}[h]
    \caption{Accuracies for different meta-learning methods on the FC100 dataset}
    \label{tab:accs:maml:fc100}
    \centering
    \begin{small}
    \begin{tabular}{ccccc}
    \toprule
               &           &       Worst acc       &       Avg acc         &       Best acc        \\
    \toprule
    
    
    
    \multirow{3}{*}{MAML}   &   1-shot      &       7.32 $\pm$ 1.49\%        &      31.89 $\pm$ 6.75\%       &       50.13 $\pm$ 6.75\%      \\
                               &   5-shot      &       7.51 $\pm$ 3.59\%        &      43.58 $\pm$ 7.61\%       &       61.64 $\pm$ 8.50\%      \\
                              &   10-shot      &       10.51 $\pm$ 2.86\%        &      44.30 $\pm$ 6.89\%       &       64.98 $\pm$ 7.27\%     \\
                                    \cmidrule(r){1-5}
                                    
    
    \multirow{3}{*}{MC}   &   1-shot      &       6.53 $\pm$ 1.56\%        &     36.56 $\pm$ 8.05\%       &       56.75 $\pm$ 5.57\%      \\
                            &   5-shot      &       5.17 $\pm$ 1.34\%        &      47.12 $\pm$ 7.02\%       &       66.33 $\pm$ 5.11\%      \\
                            &   10-shot      &       8.35 $\pm$ 3.85\%        &      49.12 $\pm$ 6.67\%       &       65.27 $\pm$ 5.58\%     \\
                                    \cmidrule(r){1-5}
                                    
    \multirow{3}{*}{ProtoNets}   &   1-shot      &       5.25 $\pm$ 1.69\%        &      37.21 $\pm$ 0.50\%       &       59.75 $\pm$ 6.39\%        \\
                              &   5-shot      &       5.57 $\pm$ 2.99\%        &      50.49 $\pm$ 0.48\%       &       70.31 $\pm$ 6.41\%      \\
                               &   10-shot      &       9.93 $\pm$ 4.39\%        &      56.15 $\pm$ 0.47\%       &       72.94 $\pm$ 6.23\%     \\
                                    \cmidrule(r){1-5}
                                    
    \multirow{3}{*}{R2D2}   &   1-shot      &       6.13 $\pm$ 1.63\%        &      37.91 $\pm$ 0.48\%       &       59.67 $\pm$ 6.17\%         \\
                       &   5-shot      &       6.72 $\pm$ 2.85\%        &      54.35 $\pm$ 0.49\%       &       74.34 $\pm$ 6.59\%      \\
                      &   10-shot      &       12.00 $\pm$ 7.07\%        &      61.72 $\pm$ 0.47\%       &       77.94 $\pm$ 3.56\%       \\
                                    \cmidrule(r){1-5}

    \multirow{3}{*}{MetaOptNet-Ridge}   &   1-shot      &       5.44 $\pm$ 1.57\%        &      39.13 $\pm$ 0.51\%       &       61.28 $\pm$ 6.18\%    \\
                               &   5-shot      &       5.97 $\pm$ 3.29\%        &      53.20 $\pm$ 0.47\%       &       72.65 $\pm$ 6.43\%        \\
                               &   10-shot      &       11.56 $\pm$ 2.78\%        &      59.52 $\pm$ 0.48\%       &       75.30 $\pm$ 1.56\%      \\
                                    \cmidrule(r){1-5}

    \multirow{2}{*}{MetaOptNet-SVM}   &   1-shot      &       5.29 $\pm$ 1.57\%        &      38.19 $\pm$ 0.48\%       &       60.14 $\pm$ 6.12\%      \\
                               &   5-shot      &       5.75 $\pm$ 2.83\%        &      54.45 $\pm$ 0.49\%       &       74.01 $\pm$ 6.64\%       \\
                              &   10-shot      &       9.54 $\pm$ 3.86\%        &      60.52 $\pm$ 0.48\%       &       77.03 $\pm$ 6.27\%       \\
    
    \bottomrule
    
    \end{tabular}
    \end{small}
\end{table}

\begin{table}[h]
    \caption{Accuracies for different meta-learning methods on the miniImageNet dataset}
    \label{tab:accs:maml:miniimagenet}
    \centering
    \begin{small}
    \begin{tabular}{ccccc}
    \toprule
               &           &       Worst acc       &       Avg acc         &       Best acc        \\
    \toprule
    
    
    
    \multirow{3}{*}{MAML}   &   1-shot      &       6.08 $\pm$ 1.77\%      &      47.13 $\pm$ 8.78\%       &       71.39 $\pm$ 6.74\%      \\
                            &   5-shot      &       10.15 $\pm$ 8.40\%        &      57.69 $\pm$ 7.92\%       &       79.60 $\pm$ 5.43\%     \\
                            &   10-shot      &       20.88 $\pm$ 7.22\%        &      59.52 $\pm$ 8.34\%       &       79.94 $\pm$ 3.41\%     \\
                                    \cmidrule(r){1-5}
                                    
    
    \multirow{2}{*}{MC}   &   1-shot      &       4.46 $\pm$ 2.05\%        &      45.03 $\pm$ 8.79\%       &       65.98 $\pm$ 6.23\%      \\
                            &   5-shot      &       5.79 $\pm$ 3.45\%        &      60.47 $\pm$ 7.57\%       &       75.09 $\pm$ 4.72\%      \\
                            &   10-shot      &       9.67 $\pm$ 3.84\%       &      60.54 $\pm$ 7.45\%       &       73.23 $\pm$ 6.24\%     \\
                                    \cmidrule(r){1-5}
                                    
    \multirow{3}{*}{ProtoNets}   &   1-shot      &       4.69 $\pm$ 2.16\%        &      53.42 $\pm$ 0.59\%       &       76.46 $\pm$ 5.64\%        \\
                               &   5-shot      &       9.53 $\pm$ 4.91\%        &      70.60 $\pm$ 0.43\%       &       85.33 $\pm$ 3.60\%      \\
                               &   10-shot      &       15.39 $\pm$ 5.17\%        &      75.91 $\pm$ 0.38\%       &       87.31 $\pm$ 3.24\%    \\
                                    \cmidrule(r){1-5}
                                    
    \multirow{3}{*}{R2D2}   &   1-shot      &       6.10 $\pm$ 3.27\%        &      56.09 $\pm$ 0.58\%       &       78.17 $\pm$ 5.33\%       \\
                               &   5-shot      &       12.03 $\pm$ 5.51\%        &      72.04 $\pm$ 0.43\%       &       86.64 $\pm$ 3.34\%       \\
                               &   10-shot      &       15.78 $\pm$ 6.10\%        &      77.32 $\pm$ 0.36\%       &       86.70 $\pm$ 1.99\%       \\
                                    \cmidrule(r){1-5}

    \multirow{2}{*}{MetaOptNet-Ridge}   &   1-shot      &       5.19 $\pm$ 3.21\%        &      57.94 $\pm$ 0.62\%       &       79.15 $\pm$ 5.05\%    \\
                               &   5-shot      &       9.83 $\pm$ 4.05\%        &      74.80 $\pm$ 0.43\%       &       84.81 $\pm$ 4.12\%        \\
                              &   10-shot      &       19.16 $\pm$ 10.07\$        &      80.31 $\pm$ 0.36\%       &       89.72 $\pm$ 2.99\%      \\
                                    \cmidrule(r){1-5}

    \multirow{2}{*}{MetaOptNet-SVM}   &   1-shot      &       4.82 $\pm$ 3.03\%        &      59.03 $\pm$ 0.62\%       &       80.38 $\pm$ 5.40\%      \\
                               &   5-shot      &       9.52 $\pm$ 4.88\%        &      75.54 $\pm$ 0.40\%       &       85.41 $\pm$ 3.86\%        \\
                              &   10-shot      &       15.93 $\pm$ 7.64\%        &      80.16 $\pm$ 0.37\%       &       90.63 $\pm$ 2.71\%       \\
    \bottomrule
    
    \end{tabular}
    \end{small}
\end{table}

\subsection{Performance range results}
\label{sec:results}
We summarize the worst, average and best accuracies of six meta-learning algorithms on three benchmark datasets (see Section \ref{sec:data} for data descriptions) in 1-shot, 5-shot, and 10-shot setting in Tables \ref{tab:accs:maml:cifar_fs}, \ref{tab:accs:maml:fc100}, and \ref{tab:accs:maml:miniimagenet}. All meta-learners are trained using code from the authors or more modern meta-learning libraries \cite{arnold2020learn2learn} (see Appendix \ref{sec:appendix:experiment:details} for implementation and additional experimental details). For evaluation we randomly partition each class in each task into 400 potential support examples composing $\cX$ and 200 query examples composing $\cD$ (all datasets have 600 examples per class). To compute average accuracy we randomly sample corresponding number of support examples per class; for the best and the worst case accuracies we use Algorithm \ref{alg:attack} to search over support examples in $\cX$.

Our key finding is the large range of performances of \emph{all} meta-learning algorithms considered in 1-, 5-, and 10-shot settings. Prototypical networks have slightly better worst-case 5-shot accuracy on CIFAR-FS, but it has large variance and is likely due to our algorithm finding poor local optima on one of the tasks. We also see no differences between meta-learners adapting end-to-end, i.e. MAML and MC, and those adapting only the last linear classification layer, i.e. R2D2 and MetaOptNet. \citet{goldblum2020unraveling} showed that the last-layer adaptation methods produce good quality linear-separable embeddings. One could expect such methods to be less sensitive to support data, however, as we discuss in Section \ref{sec:margins}, linear separability is not sufficient in the few-shot learning setting. 10-shot worst-case accuracies are not as poor (despite mostly remaining worse than random guessing), but we expect that with more support data available, the gap would narrow.\footnote{With enough support data, last-layer methods should succeed as their embeddings are linearly separable.} Finally, we note that the best-case accuracy is significantly better, especially in the 1-shot setting.

\subsection{Worst case support examples are not artifacts}
\label{sec:artifacts}

We have demonstrated that it is possible to find support examples yielding poor post-adaptation performance of a variety of meta-learners. It is also important to understand the nature of these worst-case examples: are they data artifacts, i.e. outliers or miss-labeled examples, or realistic images that could cause failures in practice? We argue that the latter is the case.

\begin{enumerate}
    \item In Figure \ref{fig:exemplars} we presented several examples of the worst-case support images on CIFAR-FS: they are correctly labeled and appear representative of the respective classes. Inspecting the images on the left closer we note that it is often not easy to notice visually that such support examples could result in a poor performance without significant expert knowledge of the dataset. Only the cellphone image labelled ``telephone'' and grey-scale image labeled ``plain'' seem potentially problematic. On the other hand, ``snail'' and ``bed'' images all appear normal. On the right figure, majority of the images also appear reasonable.
    \item 10-shot setting should be a lot more resilient to outliers, however Algorithm \ref{alg:attack} continues to be successful in 10-shot setting, finding ten different examples per class leading to accuracy slightly better than a random predictor as shown in Tables \ref{tab:accs:maml:cifar_fs}, \ref{tab:accs:maml:fc100}, and \ref{tab:accs:maml:miniimagenet}.
    \item In Figure \ref{fig:app:accs-attack} we present histograms of accuracies visualizing the first iteration over classes of Algorithm \ref{alg:attack} in 1-shot learning on CIFAR-FS. The right most histogram corresponds to post-adaptation accuracies for different choices of support image for class 0 and random choices for classes 1-4. The subsequent histogram is for different choices of support images for class 1, where image for class 0 is chosen with Algorithm \ref{alg:attack} and classes 2-4 are random, and analogously for the remaining three histograms. We see that the lower accuracy tails of the first two histograms contain multiple worst-case support examples in the corresponding classes. By the third histogram, the range of accuracies is below 50\% for \emph{all} possible support examples in the corresponding classes.
    \item In Figure \ref{fig:app:test_acc_hist} we present histogram of accuracies of 3991 unique combinations from CIFAR-FS of 1-shot support examples evaluated by Algorithm \ref{alg:attack} throughout 3 iterations. There are 3335 distinct sets of 5 examples each with less than 50\% post-adaptation accuracy.
\end{enumerate}
We present analogous analysis for other meta-learners and datasets in Appendix \ref{sec:appendix:worstcase_examples}, where we also conclude that worst-case adaptation examples are realistic and can cause malfunctions in practice.

\begin{figure}
    \begin{minipage}[c]{0.48\textwidth}
    \centering
    \includegraphics[width=1.0\linewidth]{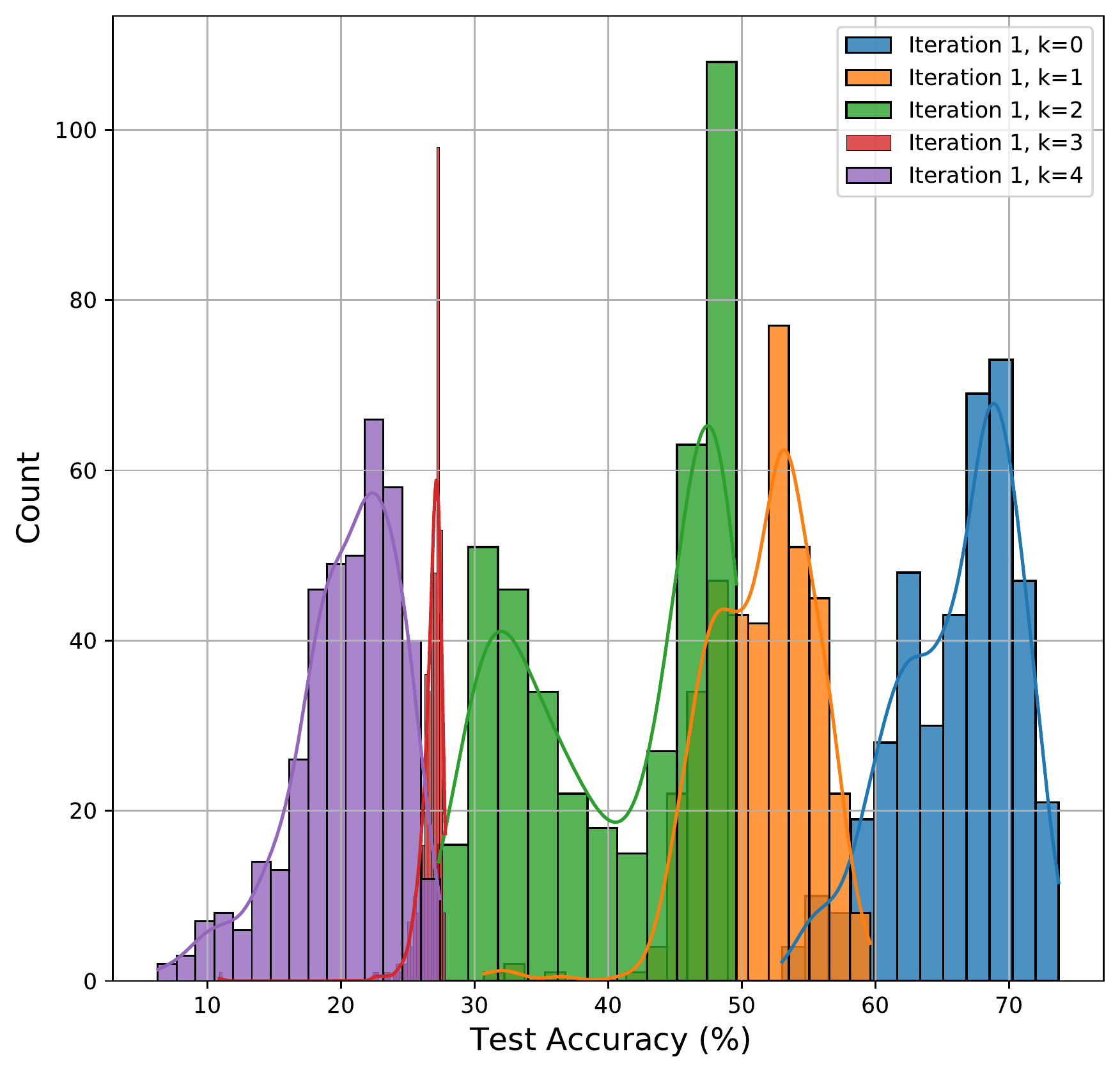}
    \vspace*{-6mm}
    \caption{Histogram of test accuracies on the first iteration of Algorithm \ref{alg:attack} as it progresses through classes evaluating 1-shot combinations of images for adaptation on a CIFAR-FS test task with MetaOptNet-SVM meta-learner.} 
    \label{fig:app:accs-attack}
  \end{minipage} \hfill
  \begin{minipage}[c]{0.48\textwidth}
    \centering
    \includegraphics[width=1.0\linewidth]{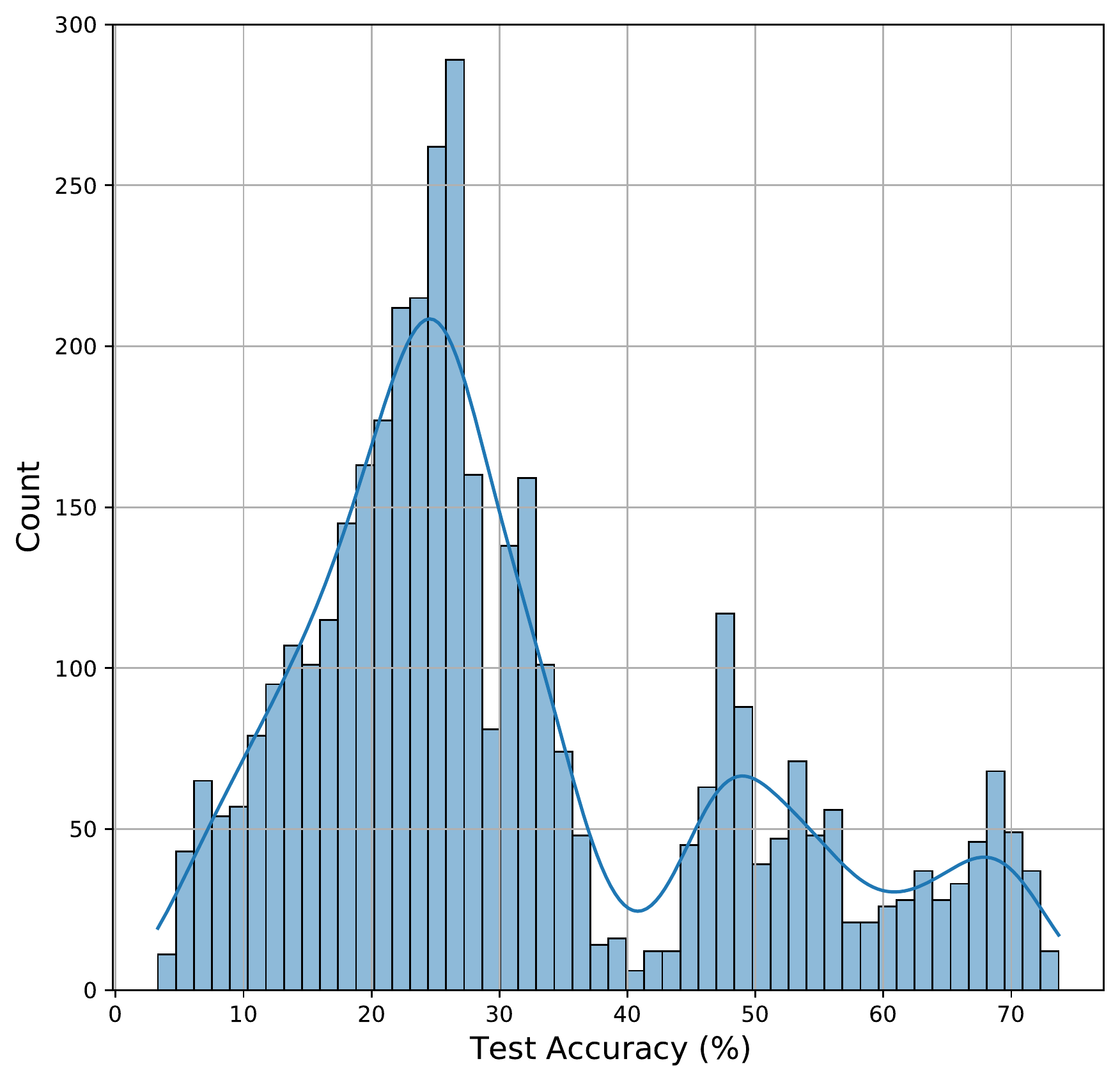}
    \vspace*{-6mm}
    \caption{Histogram of test accuracies computed during 3 iterations of Algorithm \ref{alg:attack} evaluating different unique 1-shot combinations of images for adaptation on a CIFAR-FS test task with MetaOptNet-SVM meta-learner.} 
    \label{fig:app:test_acc_hist}
  \end{minipage}
\end{figure}

\subsection{Improving support data robustness with adversarial training}
\label{sec:advtraining}
The issue of robustness has been studied in many contexts in machine learning, and adversarial vulnerability of deep learning based models has been explored extensively in the recent literature \citep{goodfellow2014explaining,carlini2017towards,madry2017towards}. While the support data sensitivity of meta learners presented in this work is a new type of non-robustness, it is possible to approach the problem borrowing ideas from adversarial training \citep{madry2017towards}. The high-level idea of adversarial training is to use a mechanism exposing non-robustness to guide the data used for model updates. For example, if a model is sensitive to perturbations in the inputs, for an incoming batch of data we find the worst-case (by maximizing the loss) perturbations to the inputs and use the perturbed inputs to updated model parameters. To converge, adversarial training needs to find model parameters such that the mechanism exposing non-robustness can no-longer damage the performance. Adversarial training has theoretical guarantees for convex models \citep{wald1945statistical} and has been shown empirically to be successful in defending against adversarial attacks on deep learning models \citep{madry2017towards}.

We use adversarial training scheme in an attempt to achieve robustness to support data in meta learning. Specifically, we use Algorithm \ref{alg:attack} to find worst-case adaptation examples during training instead of using random ones as in standard training. The result is quite intriguing: Tables \ref{tab:cifarfs-adversarial} and \ref{tab:fc100-adversarial} summarize performance of adversarially (in a sense of support data) trained meta-learners on train and test tasks. In most cases, adversarial training converged, i.e. we are no longer able to find detrimental worst-case support examples with Algorithm \ref{alg:attack} on the \emph{training} tasks, however we observe no improvements of the worst-case accuracy on the \emph{test} tasks. Our experiment demonstrates that support data sensitivity in meta-learning is not easily addressed with existing methods and requires exploring new solution paths.

\begin{table}[h]

\caption{Accuracies for different meta-learning methods trained in the standard manner and adversarially on the CIFAR-FS dataset}
\label{tab:cifarfs-adversarial}
\centering
\begin{small}
\begin{tabular}{cccccc}
\toprule
     Method                     &     Dataset                   &   Training          & Worst acc           & Avg acc            & Best acc           \\
\toprule
\multirow{4}{*}{R2D2}           & \multirow{2}{*}{Train} & Standard      & 13.83 $\pm$ 9.35 \% & 87.68 $\pm$ 0.56\% & 96.91 $\pm$ 3.15\% \\
                                &                               & Adversarial & 46.34 $\pm$ 14.93\% & 88.19 $\pm$ 0.59\%  & 97.94 $\pm$ 2.89\% \\
                                \cmidrule(r){2-6}
                                & \multirow{2}{*}{Test}     & Standard      & 6.14 $\pm$ 2.89\%   & 68.86 $\pm$ 0.68\% & 86.63 $\pm$ 5.97\% \\
                                &                               & Adversarial & 6.76 $\pm$ 2.69\%   & 68.62 $\pm$ 0.66\% & 87.46 $\pm$ 5.86\% \\
\midrule
\multirow{4}{*}{MetaOptNet-Ridge} & \multirow{2}{*}{Train} & Standard      & 77.75 $\pm$ 17.80\% & 99.23 $\pm$ 0.15\% & 99.81 $\pm$ 0.99\% \\
                                &                               & Adversarial & 90.08 $\pm$ 14.22\% & 98.87 $\pm$ 0.21\% & 99.84 $\pm$ 0.66\% \\
                                \cmidrule(r){2-6}
                                & \multirow{2}{*}{Test}     & Standard      & 5.17 $\pm$ 2.96\%   & 71.21 $\pm$ 0.67\% & 87.41 $\pm$ 5.95\% \\
                                &                               & Adversarial & 5.38 $\pm$ 2.79\%   & 71.81 $\pm$ 0.67\% & 88.42 $\pm$ 5.60\% \\
\midrule
\multirow{4}{*}{MetaOptNet-SVM}   & \multirow{2}{*}{Train} & Standard      & 9.74 $\pm$ 9.00\%   & 91.93 $\pm$ 0.47\% & 97.58 $\pm$ 2.51\% \\
                                &                               & Adversarial & 93.47 $\pm$ 11.84\% & 99.08 $\pm$ 0.19\% & 99.81 $\pm$ 0.89\% \\
                                \cmidrule(r){2-6}
                                & \multirow{2}{*}{Test}     & Standard      & 5.27 $\pm$ 2.82\%   & 70.79 $\pm$ 0.69\% & 87.66 $\pm$ 5.76\% \\
                                &                               & Adversarial & 4.97 $\pm$ 2.58\%   & 71.11 $\pm$ 0.70\% & 87.84 $\pm$ 5.93\% \\
\bottomrule
\end{tabular}
\end{small}
\end{table}

\begin{table}[h]

\caption{Accuracies for different meta-learning methods trained in the standard manner and adversarially on the FC100 dataset}
\label{tab:fc100-adversarial}
\centering
\begin{small}
\begin{tabular}{cccccc}
\toprule
    Method                       &      Dataset                 &     Training        & Worst acc            & Avg acc            & Best acc           \\
\toprule
\multirow{4}{*}{R2D2}           & \multirow{2}{*}{Train} & Standard       & 8.95 $\pm$ 6.62\%    & 84.24 $\pm$ 0.60\% & 95.48 $\pm$ 3.45\% \\
                                &                               & Adversarial & 45.01 $\pm$ 12.55\%  & 87.50 $\pm$ 0.57\% & 97.11 $\pm$ 2.97\% \\
                                \cmidrule(r){2-6}
                                & \multirow{2}{*}{Test}     & Standard       & 6.13 $\pm$ 1.63\%    & 37.91 $\pm$ 0.48\% & 59.67 $\pm$ 6.17\% \\
                                &                               & Adversarial  & 6.44 $\pm$ 1.68\%    & 38.70 $\pm$ 0.46\% & 60.87 $\pm$ 6.30\% \\
\midrule
\multirow{4}{*}{MetaOptNet-Ridge} & \multirow{2}{*}{Train} & Standard       & 14.27 $\pm$ 11.32 \% & 92.54 $\pm$ 0.45\% & 97.62 $\pm$ 2.51\% \\
                                &                               & Adversarial  & 79.67 $\pm$ 14.69\%  & 96.83 $\pm$ 0.38\% & 98.76 $\pm$ 2.26\% \\
                                \cmidrule(r){2-6}
                                & \multirow{2}{*}{Test}     & Standard       & 5.44 $\pm$ 1.57\%    & 39.13 $\pm$ 0.51\% & 61.29 $\pm$ 6.18\% \\
                                &                               & Adversarial  & 5.47 $\pm$ 1.89\%    & 36.87 $\pm$ 0.48\% & 59.07 $\pm$ 6.63\% \\
\midrule
\multirow{4}{*}{MetaOptNet-SVM}   & \multirow{2}{*}{Train} & Standard       & 18.05 $\pm$ 12.64\%  & 94.13 $\pm$ 0.38\% & 98.28 $\pm$ 2.09\% \\
                                &                               & Adversarial  & 90.37 $\pm$ 12.18\%  & 97.12 $\pm$ 0.39\% & 99.06 $\pm$ 1.98\% \\
                                \cmidrule(r){2-6}
                                & \multirow{2}{*}{Test}     & Standard       & 5.29 $\pm$ 1.57\%    & 38.19 $\pm$ 0.48\% & 60.14 $\pm$ 6.11\% \\
                                &                               & Adversarial  & 6.08 $\pm$ 1.78\%    & 35.92 $\pm$ 0.45\% & 60.30 $\pm$ 6.26\% \\
\bottomrule
\end{tabular}
\end{small}
\end{table}


\section{Meta-learning margin analysis}
\label{sec:margins}
\begin{figure}[!h]
  \centering
  
  \subfigure[Train task embeddings (standard training)]{\includegraphics[width=0.24\textwidth]{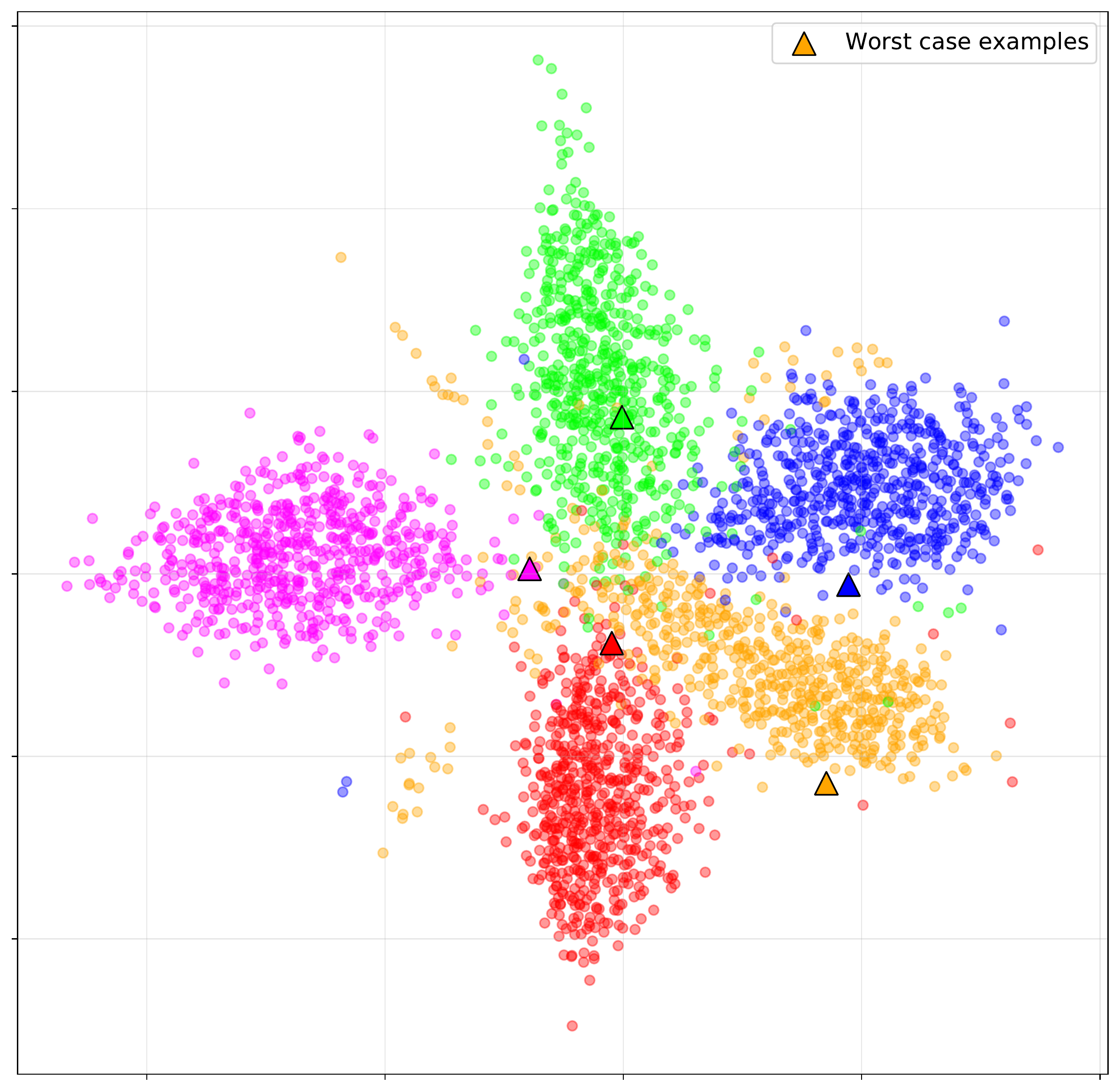}} \hfill
  \subfigure[Test task embeddings (standard training)]{\includegraphics[width=0.24\textwidth]{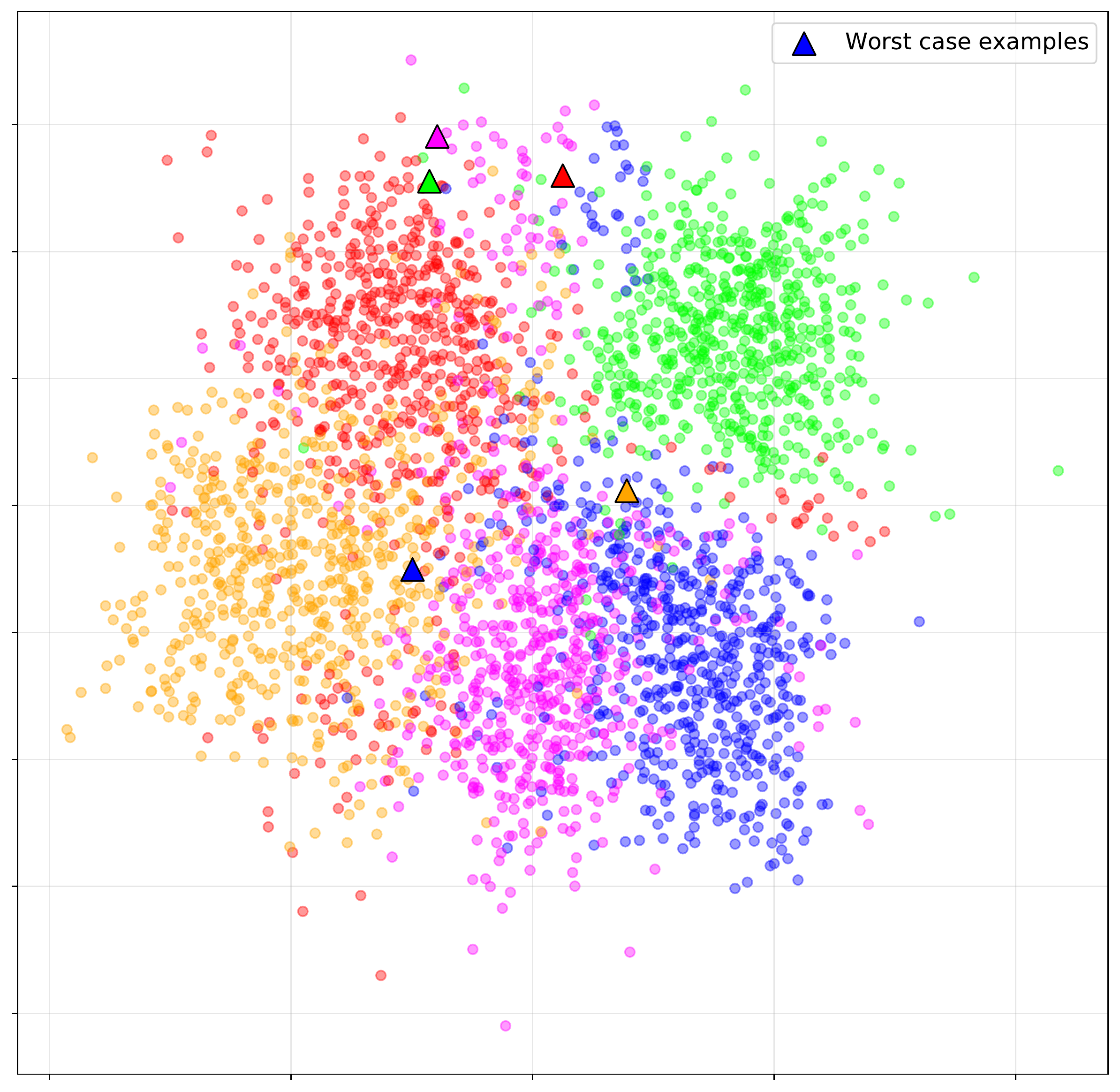}}\hfill
  \subfigure[Train task embeddings (adversarial training)]{\includegraphics[width=0.24\textwidth]{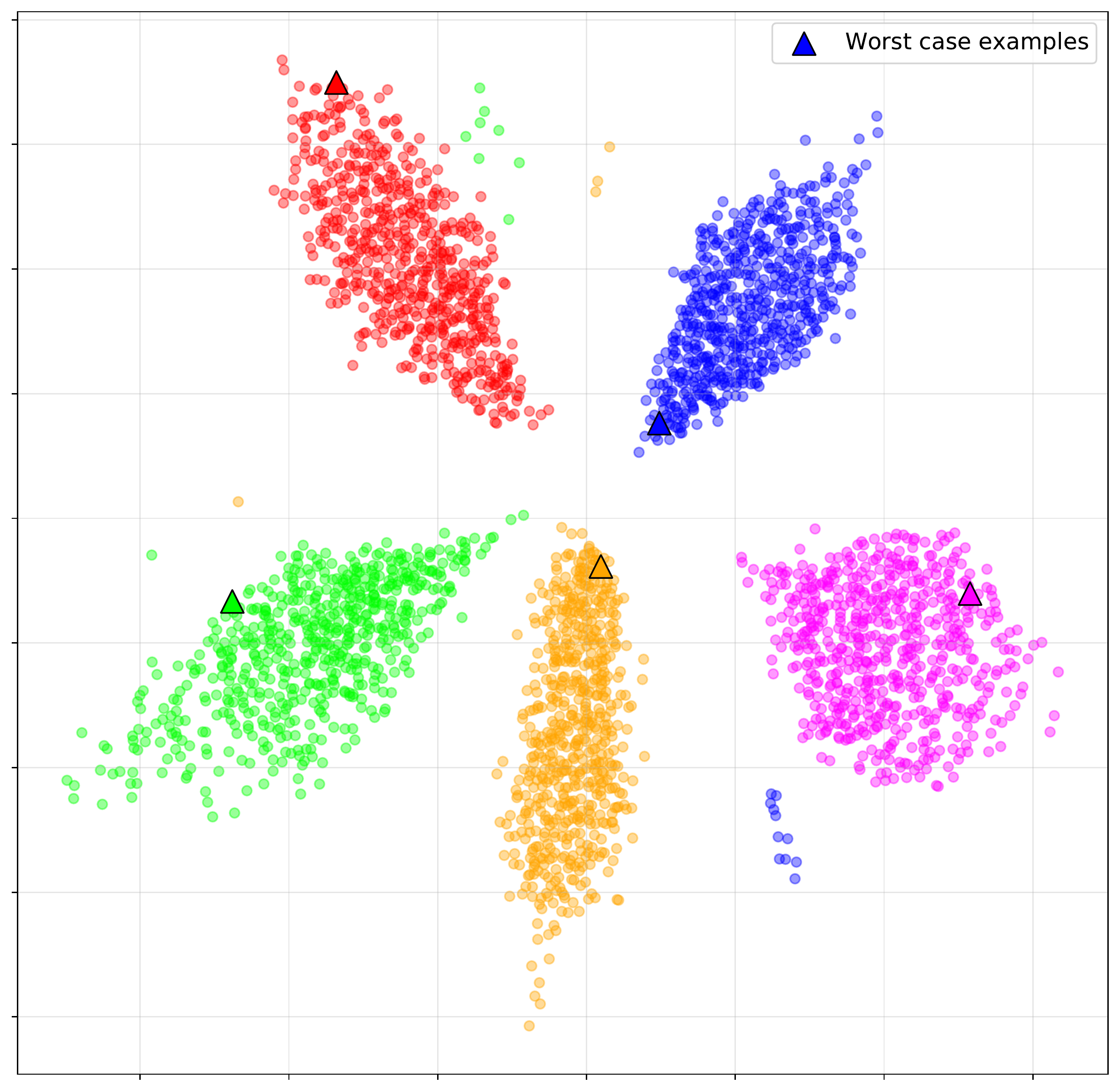}} \hfill
  \subfigure[Test task embeddings (adversarial training)]{\includegraphics[width=0.24\textwidth]{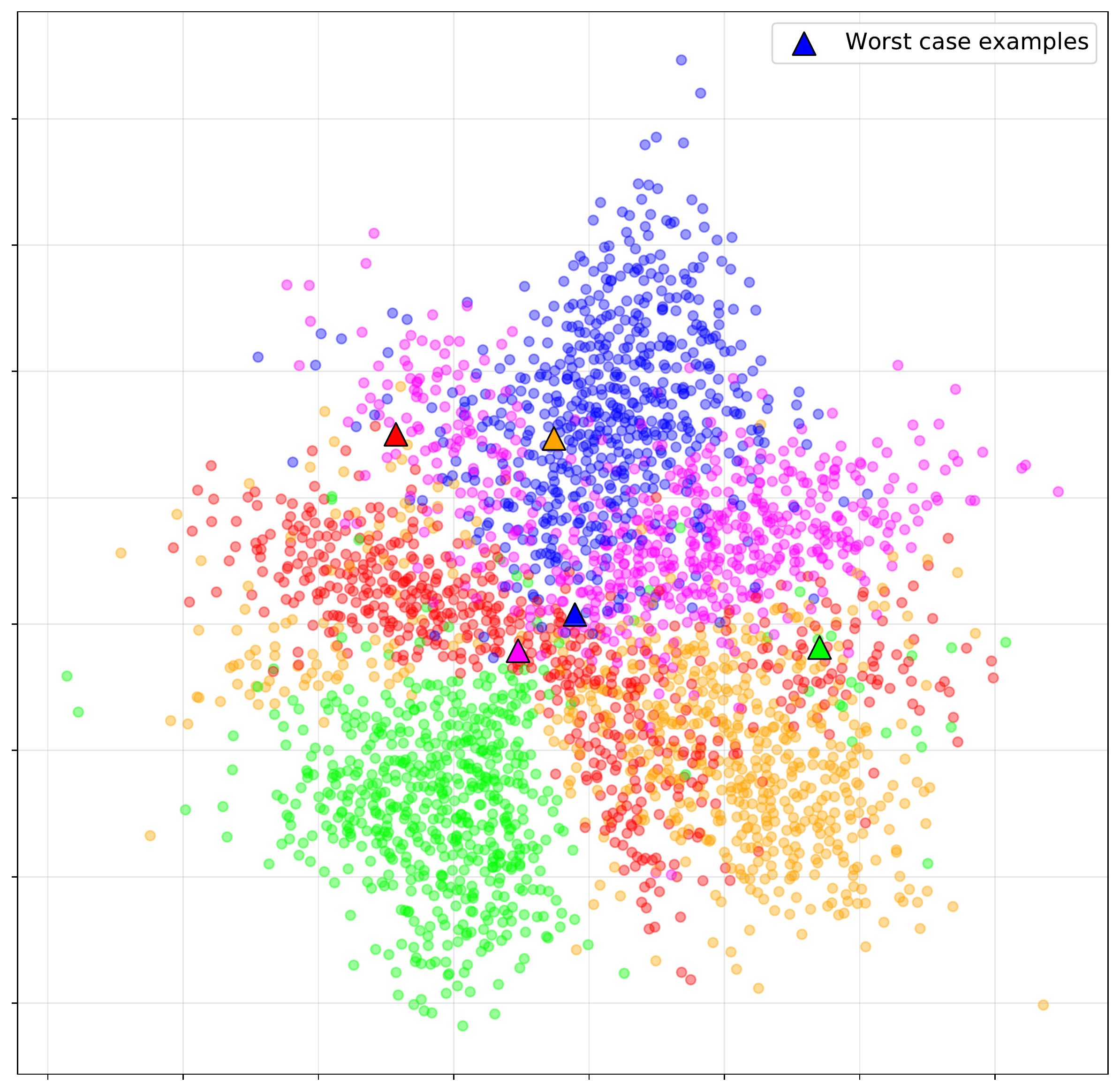}}\hfill
  
  \caption{Projected embeddings of MetaOptNet-SVM for a train and a test task query data from CIFAR-FS. We compare standard and adversarial training discussed in Section \ref{sec:advtraining}. Points are colored with their labels. Highlighted points are the worst-case support examples selected with Algorithm \ref{alg:attack}.
}
  \label{fig:embeddings}
\end{figure}

Meta-learners, specifically those only adapting the last layer linear classifier, produce embeddings that appear (approximately) linearly separable even when projected into two dimensions as shown in Figure \ref{fig:embeddings} (we used Multidimensional Scaling \citep{mead1992review} to preserve the relative cluster sizes and distances between clusters). In the supervised learning context this would be considered an easy problem for, e.g., linear SVM as in the MetaOptNet-SVM meta-learner. The problem, however, is that in the supervised learning context we typically have sufficient data, while in the few-shot setting meta-learners are restricted to as little as a single example per class in a high-dimensional space.

Lets take another look at Figure \ref{fig:embeddings}. These are Multidimensional Scaling \citep{mead1992review} projections of the embeddings obtained with the MetaOptNet trained on CIFAR-FS. Embeddings in (a) and (c) correspond to a query data from a train task for standard and adversarially trained models, and embeddings in (b) and (d) to a query data from a test task for the corresponding models. All embeddings appear well-clustered, however embeddings in (c) have the largest separation and the smallest within-class variance. This wouldn't make a significant difference in supervised learning with enough data, but makes a big difference for meta-learning: when using Algorithm \ref{alg:attack} to find the worst-case support examples (highlighted in the figures), the corresponding accuracies are 1.8\% for (a), 2.90\% for (b), 99.4\% for (c), and 5.70\% for (d).
We see that although all embeddings are well-separated\footnote{For comparison, the accuracies of the corresponding supervised learning problems (i.e. linear classifiers trained using embeddings of all 400 per class potential support examples $\cX$, rather than a single example per class) are 99.6\% for (a) and 94.3\% for (b).}, only (c) is robust to support data selection.
We present the projected embeddings for a variety of meta-learners on the FC100 dataset in Appendix \ref{sec:appendix:embeddings}.
As we discuss next, robust meta-learning, in addition to vanilla linear-separability, requires features with bigger class separation and lower intra-class variances.

In the following theorem, we show that as long as the class embeddings are sufficiently separated, the probability of any two points sampled from the classes leading to a max-margin classifier with large misclassification rate is exponentially small. We note the high degree of separation (linear in embedding dimension) necessary to guarantee robustness to the choice of support data. This suggests unless the class embeddings are well-separated, the resulting meta-learning algorithm will be sensitive to the choice of support data.

\begin{theorem}
Consider a (binary) Gaussian discriminant analysis model: 
\[
\begin{aligned}
X\mid Y=1 \sim N(\mu, \sigma^2I_d), \\
X\mid Y=0 \sim N(-\mu,\sigma^2I_d).
\end{aligned}
\]
As long as $\mu = (d+\sqrt{2dt} + 2t)\sigma$ for some $t>0$, then the max-margin classifier between two points sampled independently from each cluster has misclassification rate at most $\Phi(-d-\sqrt{2dt} - 2t)$ with probability at least $(1-e^{-t})^2$.
\end{theorem}

\begin{proof}
Consider a (binary) Gaussian discriminant analysis model:
\[
\begin{aligned}
X\mid Y=1 \sim N(\mu, \sigma^2I_d), \\
X\mid Y=0 \sim N(-\mu,\sigma^2I_d).
\end{aligned}
\]
Define the core of the clusters as the sets 
\[
\begin{aligned}
\cC_1 \triangleq \{x\in\reals^d\mid\|x -\mu\|_2 \le r\sigma\}, \\
\cC_0 \triangleq \{x\in\reals^d\mid\|x + \mu\|_2 \le r\sigma\}.
\end{aligned}
\]
It is a tedious geometric exercise to show that as long as $\mu > 2r\sigma$, then the midpoint of any pair of points $(x_1,x_0)\in\cC_1\times\cC_0$ falls outside $\cC_1\cup\cC_0$. Further, the hyperplane 
\[\textstyle
\cH\triangleq\{x\in\reals^d\mid (x_1-x_0)^Tx = \frac12(\|x_1\|_2^2 - \|x_0\|_2^2)\}
\]
bisects $\cC_0$ and $\cC_1$ for any choice of $x_1,x_0$. This implies the risk of any max-margin classifier constructed from $x_1$ and $x_0$ has misclassification error rate at most $2\Phi(-r)$, where $\Phi$ is the $N(0,1)$ CDF. We note that the probability of a pair of independently sampled points $x_1\sim N(\mu,\sigma^2I_d)$ and $x_0\sim N(-\mu,\sigma^2I_d)$ falling in $\cC_1$ and $\cC_0$ respectively is $F_{\chi^2_d}(r)$, where $F_{\chi^2_d}$ is the CDF of a $\chi^2_d$ random scalar. By picking $r = d + \sqrt{2dt} + 2t$ for some $t>0$, the probability of $x_1\in\cC_1$ and $x_0\in\cC_0$ is at least $(1-e^{-t})^2$.
\end{proof}



\section{Conclusion}
\label{sec:conclusion}
We studied the problem of support data sensitivity in meta-learning: the performance of existing algorithms is extremely sensitive to the examples used for adaptation. Our findings suggest that when deploying meta-learning, especially in safety-critical applications such as autonomous driving or medical imaging, practitioners should carefully check the support examples they label for the meta-learner. However, even when the data is interpretable for a human, e.g. images, recognizing potentially detrimental examples could be hard as we have seen in our experiments.

In our experiments, we considered popular few-shot image classification benchmarks. We note that meta-learning has also been applied to data in other modalities such as language understanding \citep{dou2019investigating} and speech recognition \citep{hsu2020meta}. We expect our conclusions and Algorithm \ref{alg:attack} for finding the worst-case support data to apply in other meta-learning applications, however, an empirical study is needed to verify this.

Going forward, our results suggest that robustness in meta-learning could be achieved by explicitly encouraging separation and tighter intra-class embeddings (at least in the context of last-layer adaptation meta-learners). Unfortunately, the adversarial training approach, while successful in promoting robustness in many applications, fails to achieve robustness of meta-learners to support data. In our experiments, adversarial training achieved well-separated and tight intra-class embeddings resulting in robustness on the train tasks (i.e., tasks composed of classes seen during training), but failed to improve on the test tasks. Our findings demonstrate that new approaches are needed to achieve robustness in meta-learning.

Finally, we note that our results provide a new perspective on a different meta-learning phenomenon studied in prior work. \citet{setlur2020support} and \citet{ni2021data} studied the importance of the support data during training. \citet{setlur2020support} quantified the impact of the support data diversity, while \citet{ni2021data} considered the effectiveness of the support data augmentation, and both concluded that meta-learning is insensitive to the support data quality and diversity. In our experiments with a variation of adversarial training in Section \ref{sec:advtraining}, where support data during training is selected based on Algorithm \ref{alg:attack} to find worst-case support examples, we achieved significant improvements in terms of the worst-case accuracies on the train tasks. Thus, the conclusion is different from \citep{setlur2020support,ni2021data}, i.e. the support data used during training has an impact on the resulting meta-learner from the perspective of sensitivity studied in our work.

\begin{ack}
This note is based upon work supported by the National Science Foundation (NSF) under grants no.\ 1916271, 2027737, and 2113373. Any opinions, findings, and conclusions or recommendations expressed in this note are those of the authors and do not necessarily reflect the views of the NSF.
\end{ack}

\bibliographystyle{plainnat}
\bibliography{MY}

\begin{thebibliography}{45}
\providecommand{\natexlab}[1]{#1}
\providecommand{\url}[1]{\texttt{#1}}
\expandafter\ifx\csname urlstyle\endcsname\relax
  \providecommand{\doi}[1]{doi: #1}\else
  \providecommand{\doi}{doi: \begingroup \urlstyle{rm}\Url}\fi

\bibitem[Amos and Kolter(2017)]{amos2017optnet}
Brandon Amos and J~Zico Kolter.
\newblock Optnet: Differentiable optimization as a layer in neural networks.
\newblock In \emph{International Conference on Machine Learning}, pages
  136--145. PMLR, 2017.

\bibitem[Angwin et~al.(2016)Angwin, Larson, Mattu, and
  Kirchner]{angwin2016Machine}
Julia Angwin, Jeff Larson, Surya Mattu, and Lauren Kirchner.
\newblock Machine {{Bias}}.
\newblock
  www.propublica.org/article/machine-bias-risk-assessments-in-criminal-sentencing,
  May 2016.

\bibitem[Arnold et~al.(2020)Arnold, Mahajan, Datta, Bunner, and
  Zarkias]{arnold2020learn2learn}
S{\'e}bastien~MR Arnold, Praateek Mahajan, Debajyoti Datta, Ian Bunner, and
  Konstantinos~Saitas Zarkias.
\newblock learn2learn: A library for meta-learning research.
\newblock \emph{arXiv preprint arXiv:2008.12284}, 2020.

\bibitem[Bertinetto et~al.(2019)Bertinetto, Henriques, Torr, and
  Vedaldi]{bertinetto2019meta}
Luca Bertinetto, Joao~F Henriques, Philip Torr, and Andrea Vedaldi.
\newblock Meta-learning with differentiable closed-form solvers.
\newblock In \emph{International Conference on Learning Representations}, 2019.

\bibitem[Carlini and Wagner(2017)]{carlini2017towards}
Nicholas Carlini and David Wagner.
\newblock Towards evaluating the robustness of neural networks.
\newblock In \emph{2017 ieee symposium on security and privacy (sp)}, pages
  39--57. IEEE, 2017.

\bibitem[Chen et~al.(2017)Chen, Liu, Li, Lu, and Song]{chen2017targeted}
Xinyun Chen, Chang Liu, Bo~Li, Kimberly Lu, and Dawn Song.
\newblock Targeted backdoor attacks on deep learning systems using data
  poisoning.
\newblock \emph{arXiv preprint arXiv:1712.05526}, 2017.

\bibitem[Deng et~al.(2009)Deng, Dong, Socher, Li, Li, and
  Fei-Fei]{deng2009imagenet}
Jia Deng, Wei Dong, Richard Socher, Li-Jia Li, Kai Li, and Li~Fei-Fei.
\newblock Imagenet: A large-scale hierarchical image database.
\newblock In \emph{2009 IEEE conference on computer vision and pattern
  recognition}, pages 248--255. Ieee, 2009.

\bibitem[Dhillon et~al.(2019)Dhillon, Chaudhari, Ravichandran, and
  Soatto]{dhillon2019baseline}
Guneet~S Dhillon, Pratik Chaudhari, Avinash Ravichandran, and Stefano Soatto.
\newblock A baseline for few-shot image classification.
\newblock \emph{arXiv preprint arXiv:1909.02729}, 2019.

\bibitem[Dou et~al.(2019)Dou, Yu, and Anastasopoulos]{dou2019investigating}
Zi-Yi Dou, Keyi Yu, and Antonios Anastasopoulos.
\newblock Investigating meta-learning algorithms for low-resource natural
  language understanding tasks.
\newblock \emph{arXiv preprint arXiv:1908.10423}, 2019.

\bibitem[Finn et~al.(2017)Finn, Abbeel, and Levine]{finn2017model}
Chelsea Finn, Pieter Abbeel, and Sergey Levine.
\newblock Model-agnostic meta-learning for fast adaptation of deep networks.
\newblock In \emph{International Conference on Machine Learning}, pages
  1126--1135. PMLR, 2017.

\bibitem[Goldblum et~al.(2019)Goldblum, Fowl, and
  Goldstein]{goldblum2019adversarially}
Micah Goldblum, Liam Fowl, and Tom Goldstein.
\newblock Adversarially robust few-shot learning: A meta-learning approach.
\newblock \emph{arXiv preprint arXiv:1910.00982}, 2019.

\bibitem[Goldblum et~al.(2020)Goldblum, Reich, Fowl, Ni, Cherepanova, and
  Goldstein]{goldblum2020unraveling}
Micah Goldblum, Steven Reich, Liam Fowl, Renkun Ni, Valeriia Cherepanova, and
  Tom Goldstein.
\newblock Unraveling meta-learning: Understanding feature representations for
  few-shot tasks.
\newblock In \emph{International Conference on Machine Learning}, pages
  3607--3616. PMLR, 2020.

\bibitem[Goodfellow et~al.(2014)Goodfellow, Shlens, and
  Szegedy]{goodfellow2014explaining}
Ian~J Goodfellow, Jonathon Shlens, and Christian Szegedy.
\newblock Explaining and harnessing adversarial examples.
\newblock \emph{arXiv preprint arXiv:1412.6572}, 2014.

\bibitem[Grefenstette et~al.(2019)Grefenstette, Amos, Yarats, Htut, Molchanov,
  Meier, Kiela, Cho, and Chintala]{grefenstette2019generalized}
Edward Grefenstette, Brandon Amos, Denis Yarats, Phu~Mon Htut, Artem Molchanov,
  Franziska Meier, Douwe Kiela, Kyunghyun Cho, and Soumith Chintala.
\newblock Generalized inner loop meta-learning.
\newblock \emph{arXiv preprint arXiv:1910.01727}, 2019.

\bibitem[Hospedales et~al.(2020)Hospedales, Antoniou, Micaelli, and
  Storkey]{hospedales2020meta}
Timothy Hospedales, Antreas Antoniou, Paul Micaelli, and Amos Storkey.
\newblock Meta-learning in neural networks: A survey.
\newblock \emph{arXiv preprint arXiv:2004.05439}, 2020.

\bibitem[Hsu et~al.(2020)Hsu, Chen, and Lee]{hsu2020meta}
Jui-Yang Hsu, Yuan-Jui Chen, and Hung-yi Lee.
\newblock Meta learning for end-to-end low-resource speech recognition.
\newblock In \emph{ICASSP 2020-2020 IEEE International Conference on Acoustics,
  Speech and Signal Processing (ICASSP)}, pages 7844--7848. IEEE, 2020.

\bibitem[Koch et~al.(2015)Koch, Zemel, and Salakhutdinov]{koch2015siamese}
Gregory Koch, Richard Zemel, and Ruslan Salakhutdinov.
\newblock Siamese neural networks for one-shot image recognition.
\newblock In \emph{ICML deep learning workshop}, volume~2. Lille, 2015.

\bibitem[Koh et~al.(2020)Koh, Sagawa, Marklund, Xie, Zhang, Balsubramani, Hu,
  Yasunaga, Phillips, Beery, et~al.]{koh2020wilds}
Pang~Wei Koh, Shiori Sagawa, Henrik Marklund, Sang~Michael Xie, Marvin Zhang,
  Akshay Balsubramani, Weihua Hu, Michihiro Yasunaga, Richard~Lanas Phillips,
  Sara Beery, et~al.
\newblock Wilds: A benchmark of in-the-wild distribution shifts.
\newblock \emph{arXiv preprint arXiv:2012.07421}, 2020.

\bibitem[Krizhevsky et~al.(2009)Krizhevsky, Hinton,
  et~al.]{krizhevsky2009learning}
Alex Krizhevsky, Geoffrey Hinton, et~al.
\newblock Learning multiple layers of features from tiny images.
\newblock 2009.

\bibitem[Lake et~al.(2011)Lake, Salakhutdinov, Gross, and
  Tenenbaum]{lake2011one}
Brenden Lake, Ruslan Salakhutdinov, Jason Gross, and Joshua Tenenbaum.
\newblock One shot learning of simple visual concepts.
\newblock In \emph{Proceedings of the annual meeting of the cognitive science
  society}, volume~33, 2011.

\bibitem[Lee et~al.(2019)Lee, Maji, Ravichandran, and Soatto]{lee2019meta}
Kwonjoon Lee, Subhransu Maji, Avinash Ravichandran, and Stefano Soatto.
\newblock Meta-learning with differentiable convex optimization.
\newblock In \emph{Proceedings of the IEEE/CVF Conference on Computer Vision
  and Pattern Recognition}, pages 10657--10665, 2019.

\bibitem[Madry et~al.(2017)Madry, Makelov, Schmidt, Tsipras, and
  Vladu]{madry2017towards}
Aleksander Madry, Aleksandar Makelov, Ludwig Schmidt, Dimitris Tsipras, and
  Adrian Vladu.
\newblock Towards deep learning models resistant to adversarial attacks.
\newblock \emph{arXiv preprint arXiv:1706.06083}, 2017.

\bibitem[Maicas et~al.(2018)Maicas, Bradley, Nascimento, Reid, and
  Carneiro]{maicas2018training}
Gabriel Maicas, Andrew~P Bradley, Jacinto~C Nascimento, Ian Reid, and Gustavo
  Carneiro.
\newblock Training medical image analysis systems like radiologists.
\newblock In \emph{International Conference on Medical Image Computing and
  Computer-Assisted Intervention}, pages 546--554. Springer, 2018.

\bibitem[Mead(1992)]{mead1992review}
Al~Mead.
\newblock Review of the development of multidimensional scaling methods.
\newblock \emph{Journal of the Royal Statistical Society: Series D (The
  Statistician)}, 41\penalty0 (1):\penalty0 27--39, 1992.

\bibitem[Miller et~al.(2000)Miller, Matsakis, and Viola]{miller2000learning}
Erik~G Miller, Nicholas~E Matsakis, and Paul~A Viola.
\newblock Learning from one example through shared densities on transforms.
\newblock In \emph{Proceedings IEEE Conference on Computer Vision and Pattern
  Recognition. CVPR 2000 (Cat. No. PR00662)}, volume~1, pages 464--471. IEEE,
  2000.

\bibitem[Mishra et~al.(2017)Mishra, Rohaninejad, Chen, and
  Abbeel]{mishra2017simple}
Nikhil Mishra, Mostafa Rohaninejad, Xi~Chen, and Pieter Abbeel.
\newblock A simple neural attentive meta-learner.
\newblock \emph{arXiv preprint arXiv:1707.03141}, 2017.

\bibitem[Ni et~al.(2021)Ni, Goldblum, Sharaf, Kong, and Goldstein]{ni2021data}
Renkun Ni, Micah Goldblum, Amr Sharaf, Kezhi Kong, and Tom Goldstein.
\newblock Data augmentation for meta-learning.
\newblock In \emph{International Conference on Machine Learning}, pages
  8152--8161. PMLR, 2021.

\bibitem[Nichol et~al.(2018)Nichol, Achiam, and Schulman]{nichol2018first}
Alex Nichol, Joshua Achiam, and John Schulman.
\newblock On first-order meta-learning algorithms.
\newblock \emph{arXiv preprint arXiv:1803.02999}, 2018.

\bibitem[Novak et~al.(1984)Novak, Gowin, and Bob]{novak1984learning}
Joseph~D Novak, D~Bob Gowin, and Gowin~D Bob.
\newblock \emph{Learning how to learn}.
\newblock cambridge University press, 1984.

\bibitem[Oldewage et~al.(2020)Oldewage, Bronskill, and
  Turner]{oldewage2020attacking}
Elre~Talea Oldewage, John~F Bronskill, and Richard~E Turner.
\newblock Attacking few-shot classifiers with adversarial support sets.
\newblock 2020.

\bibitem[Oreshkin et~al.(2018)Oreshkin, Rodriguez, and
  Lacoste]{oreshkin2018tadam}
Boris~N Oreshkin, Pau Rodriguez, and Alexandre Lacoste.
\newblock {TADAM}: task dependent adaptive metric for improved few-shot
  learning.
\newblock In \emph{Proceedings of the 32nd International Conference on Neural
  Information Processing Systems}, pages 719--729, 2018.

\bibitem[Park and Oliva(2019)]{NEURIPS2019_57c0531e}
Eunbyung Park and Junier~B Oliva.
\newblock Meta-curvature.
\newblock In \emph{Advances in Neural Information Processing Systems},
  volume~32, 2019.
\newblock URL
  \url{https://proceedings.neurips.cc/paper/2019/file/57c0531e13f40b91b3b0f1a30b529a1d-Paper.pdf}.

\bibitem[Qiao et~al.(2018)Qiao, Liu, Shen, and Yuille]{qiao2018few}
Siyuan Qiao, Chenxi Liu, Wei Shen, and Alan~L Yuille.
\newblock Few-shot image recognition by predicting parameters from activations.
\newblock In \emph{Proceedings of the IEEE Conference on Computer Vision and
  Pattern Recognition}, pages 7229--7238, 2018.

\bibitem[Ravi and Larochelle(2016)]{ravi2016optimization}
Sachin Ravi and Hugo Larochelle.
\newblock Optimization as a model for few-shot learning.
\newblock 2016.

\bibitem[Sallab et~al.(2017)Sallab, Saeed, Tawab, and Abdou]{sallab2017meta}
Ahmad~El Sallab, Mahmoud Saeed, Omar~Abdel Tawab, and Mohammed Abdou.
\newblock Meta learning framework for automated driving.
\newblock \emph{arXiv preprint arXiv:1706.04038}, 2017.

\bibitem[Setlur et~al.(2020)Setlur, Li, and Smith]{setlur2020support}
Amrith Setlur, Oscar Li, and Virginia Smith.
\newblock Is support set diversity necessary for meta-learning?
\newblock \emph{arXiv preprint arXiv:2011.14048}, 2020.

\bibitem[Slack et~al.(2020)Slack, Friedler, and Givental]{slack2020fairness}
Dylan Slack, Sorelle~A Friedler, and Emile Givental.
\newblock Fairness warnings and fair-maml: learning fairly with minimal data.
\newblock In \emph{Proceedings of the 2020 Conference on Fairness,
  Accountability, and Transparency}, pages 200--209, 2020.

\bibitem[Snell et~al.(2017)Snell, Swersky, and Zemel]{snell2017prototypical}
Jake Snell, Kevin Swersky, and Richard~S Zemel.
\newblock Prototypical networks for few-shot learning.
\newblock \emph{arXiv preprint arXiv:1703.05175}, 2017.

\bibitem[Sohn et~al.(2020)Sohn, Berthelot, Li, Zhang, Carlini, Cubuk, Kurakin,
  Zhang, and Raffel]{sohn2020fixmatch}
Kihyuk Sohn, David Berthelot, Chun-Liang Li, Zizhao Zhang, Nicholas Carlini,
  Ekin~D Cubuk, Alex Kurakin, Han Zhang, and Colin Raffel.
\newblock Fixmatch: Simplifying semi-supervised learning with consistency and
  confidence.
\newblock \emph{arXiv preprint arXiv:2001.07685}, 2020.

\bibitem[Song et~al.(2020)Song, Yang, Choromanski, Caluwaerts, Gao, Finn, and
  Tan]{song2020rapidly}
Xingyou Song, Yuxiang Yang, Krzysztof Choromanski, Ken Caluwaerts, Wenbo Gao,
  Chelsea Finn, and Jie Tan.
\newblock Rapidly adaptable legged robots via evolutionary meta-learning.
\newblock \emph{arXiv preprint arXiv:2003.01239}, 2020.

\bibitem[Vinyals et~al.(2016)Vinyals, Blundell, Lillicrap, Kavukcuoglu, and
  Wierstra]{vinyals2016matching}
Oriol Vinyals, Charles Blundell, Timothy Lillicrap, Koray Kavukcuoglu, and Daan
  Wierstra.
\newblock Matching networks for one shot learning.
\newblock \emph{arXiv preprint arXiv:1606.04080}, 2016.

\bibitem[Wald(1945)]{wald1945statistical}
Abraham Wald.
\newblock Statistical decision functions which minimize the maximum risk.
\newblock \emph{Annals of Mathematics}, pages 265--280, 1945.

\bibitem[Wortsman et~al.(2019)Wortsman, Ehsani, Rastegari, Farhadi, and
  Mottaghi]{wortsman2019learning}
Mitchell Wortsman, Kiana Ehsani, Mohammad Rastegari, Ali Farhadi, and Roozbeh
  Mottaghi.
\newblock Learning to learn how to learn: Self-adaptive visual navigation using
  meta-learning.
\newblock In \emph{Proceedings of the IEEE/CVF Conference on Computer Vision
  and Pattern Recognition}, pages 6750--6759, 2019.

\bibitem[Xu et~al.(2020)Xu, Li, Liu, Liu, and Tang]{xu2020yet}
Han Xu, Yaxin Li, Xiaorui Liu, Hui Liu, and Jiliang Tang.
\newblock Yet meta learning can adapt fast, it can also break easily.
\newblock \emph{arXiv preprint arXiv:2009.01672}, 2020.

\bibitem[Yin et~al.(2018)Yin, Tang, Xu, and Wang]{yin2018adversarial}
Chengxiang Yin, Jian Tang, Zhiyuan Xu, and Yanzhi Wang.
\newblock Adversarial meta-learning.
\newblock \emph{arXiv preprint arXiv:1806.03316}, 2018.

\end{thebibliography}


\clearpage

\appendix

\section{Experiment details}
\label{sec:appendix:experiment:details}

\paragraph{Network architectures:} For MAML and Meta-Curvature experiments, we use a 4-layer CNN network, where each convolutional block in the network is a sequential composition of a [2 $\times$ 2 max-pooling layer, batch normalization, and a 3 $\times$ 3 convolution layer]. The final classification layer of the network is fully-connected layer mapping the input features to 5-way output.

For ProtoNet and R2D2 experiments, we use the same architectures as are used in the original papers. The ProtoNet feature extractor is a combination of four convolutional blocks. Each block consists of a 64-filter 3 × 3 convolution, batch normalization layer, a ReLU nonlinearity, and a 2 × 2 max-pooling layer. The R2D2 feature extractor is a combination of 4 convolutional layers with [96, 192, 384, 512] filters. Each convolutional layer consists of a 3 × 3 convolution, batch normalization, 2 × 2 max pooling, and a leaky ReLU with a factor of 0.1

For MetaOptNet experiments, we use the same implementation and setting as described in the original paper. MetaOptNet networks consist of a ResNet-12 network as feature extractors, and either a support vector machine or ridge regression based head for classification.

\paragraph{Meta-learning setup:} MetaOptNet networks utilize SGD with Nesterov momentum of 0.9 and a weight decay of $5 \times 10^{-4}$ for optimization. The learning rate for this set of experiments was initially set to 1.0 and then modified to 0.06 for epochs 20 to 40, 0.012 for epochs 40 to 50, and 0.024 thereafter. MAML and Meta-Curvature networks are trained using Adam optimizer with an initial learning rate of $3 \times 10^{-4}$ and $0.01$ respectively.

All networks are trained for 60000 iterations -- 60 epochs of 1000 episodes each, with the batch sizes for MAML experiments set to 32 tasks in each batch, batch size for Meta-Curvature set to 16 tasks in each batch, and for MetaOptNet experiments it's set to 8 tasks in each batch.

During the meta-training phase, we apply the random crop, color jitter, and random horizontal flip transformations for MetaOptNet networks. Additionally, we match the meta-training shot with the meta-testing shot for all networks. While meta-training, we compute the accuracy on a 5-shot 5-way validation dataset, and select the model with the best accuracy on this validation dataset for sensitivity analysis.

During evaluation, i.e. results in Tables \ref{tab:accs:maml:cifar_fs}, \ref{tab:accs:maml:fc100}, and \ref{tab:accs:maml:miniimagenet}, to compute best and worst accuracies we randomly partition each class in each task into 400 potential adaptation examples composing $\cX$ and 200 evaluation examples composing $\cD$ (all datasets have 600 examples per class). We use the corresponding algorithm to find the adaptation examples in $\cX$ and report mean and standard deviations of evaluation data $\cD$ post-adaptation accuracies over 500 random tasks for 1-shot setting, and 100 random tasks for 5-shot and 10-shot settings. To compute average accuracies we follow the setup of previous meta-learning papers, i.e. sample corresponding number of adaptation examples randomly and choose a random subset of 50 examples per class for evaluation, and report mean and standard deviation of accuracies over 1000 random tasks.

\paragraph{Adversarial training setup:} To adversarially train the models as described in Section \ref{sec:advtraining}, we initialize with models trained in the standard fashion. We then train these models in an adversarial manner for 60 epochs of 1000 episodes each. We use the same hyperparameters and experimental setup as we did for the standard training, except that we reduce the learning rate by a factor of 10. Thus, the learning rate is initially set to 0.1 for the first 20 epochs, then modified to 0.006 for epochs 20 to 40, 0.0012 for epochs 40 to 50, and 0.0024 thereafter. To find the adversarial examples, we find the worst-case examples using algorithm \ref{alg:attack} run for 3 iterations. We then use these worst-case examples as the query data to update the model parameters.

\paragraph{Algorithm runtimes:} We run all our experiments on a 12 CPU core, 32 GB RAM, and 1 V100 GPU machine. The run times for a single iteration (in our experiments we ran 3 iterations to find the worst/best case examples) for MetaOptNet-SVM method on CIFAR-FS are approximately 3 minutes for 1-shot setting, approximately 18 minutes for 5-shot setting, and approximately 43 minutes for 10-shot setting. The corresponding run times for a single iteration for R2D2 method on FC100 are approximately 1 minute for 1-shot setting, approximately 6 minutes for 5-shot setting, and approximately 20 minutes for 10-shot setting.

\section{Convergence of the algorithm to find adaptation vulnerabilities}
\label{sec:appendix:convergence}

To find the worst-case support examples (Section \ref{sec:adaptation}) and to find the adversarial examples for adversarial training (Section \ref{sec:advtraining}), we use Algorithm \ref{alg:attack} with 3 rounds of attack or iterations. In this section, we show the convergence of Algorithm \ref{alg:attack} and the rationale behind choosing 3 iterations. We execute algorithm \ref{alg:attack} to find the worst-case 5-way 10-shot support examples on CIFAR-FS and FC100 datasets for R2D2, ResNet-Ridge, and the ResNet-SVM algorithms. We track the worst-case accuracy as it updates through 10 iterations, and show the mean worst-case accuracy through the iterations averaged over 5 different randomly-sampled tasks in figure \ref{fig:appendix:convergence}. We see that while the worst-case accuracy drops significantly in the first few iterations, it stabilizes after 3 iterations and does not show significant change after the 3 iterations for all datasets and algorithms. Additionally, running for longer iterations can reduce accuracy slightly more for the 10-shot setting as compared to the 1-shot and 5-shot settings; however, the 10-shot setting is more robust to this algorithm than the 5-shot setting and thus the drop in accuracy in later iterations is understandable.

\begin{figure}[h]
  \centering
  
  \subfigure[Worst-case accuracy over 10 iterations of algorithm \ref{alg:attack} for different algorithms and on the CIFAR-FS dataset. ]{\includegraphics[width=0.4\textwidth]{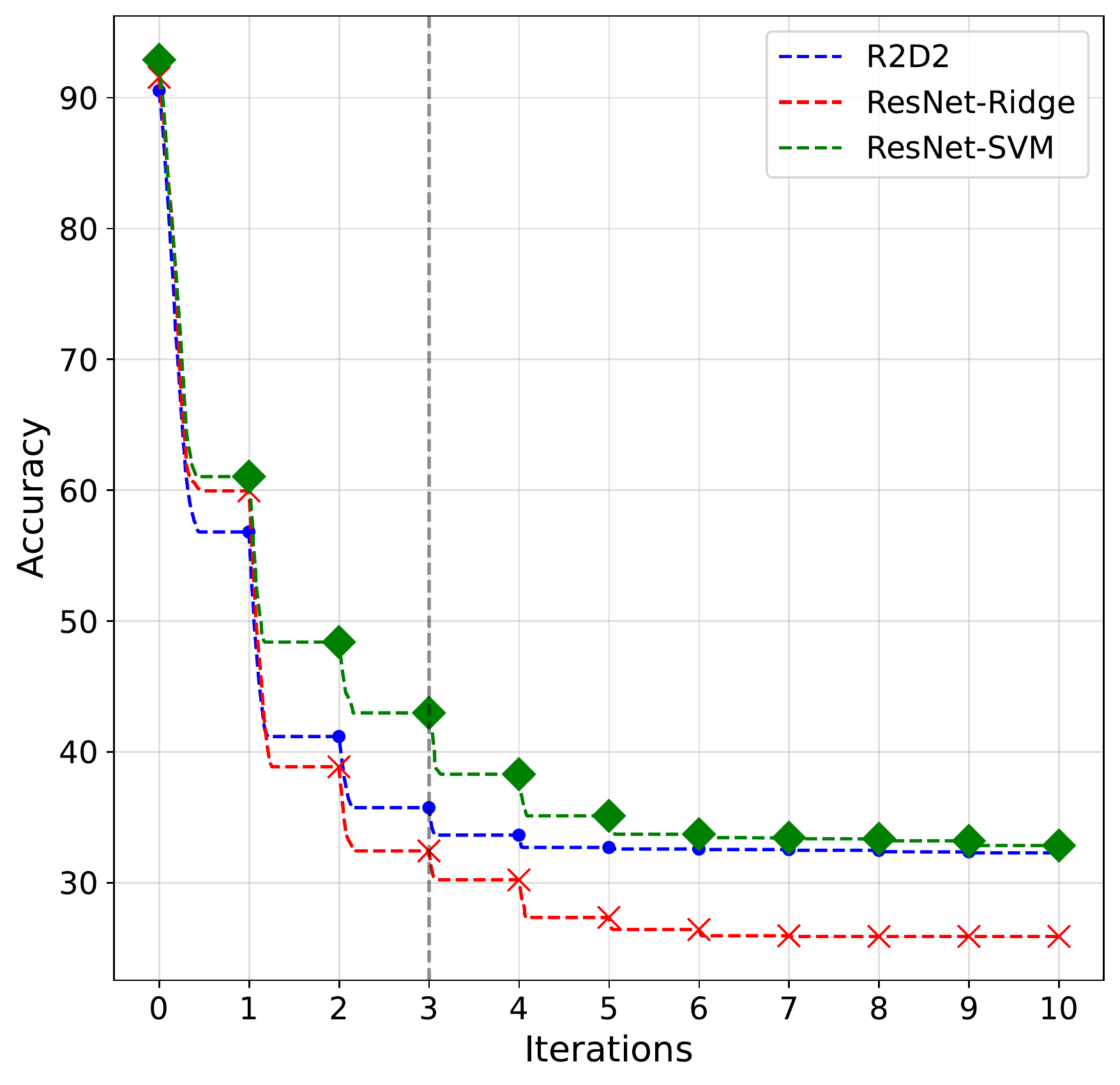}} 
  \hfill
  \subfigure[Worst-case accuracy over 10 iterations of Algorithm \ref{alg:attack} for different algorithms and on the FC100 dataset. ]{\includegraphics[width=0.4\textwidth]{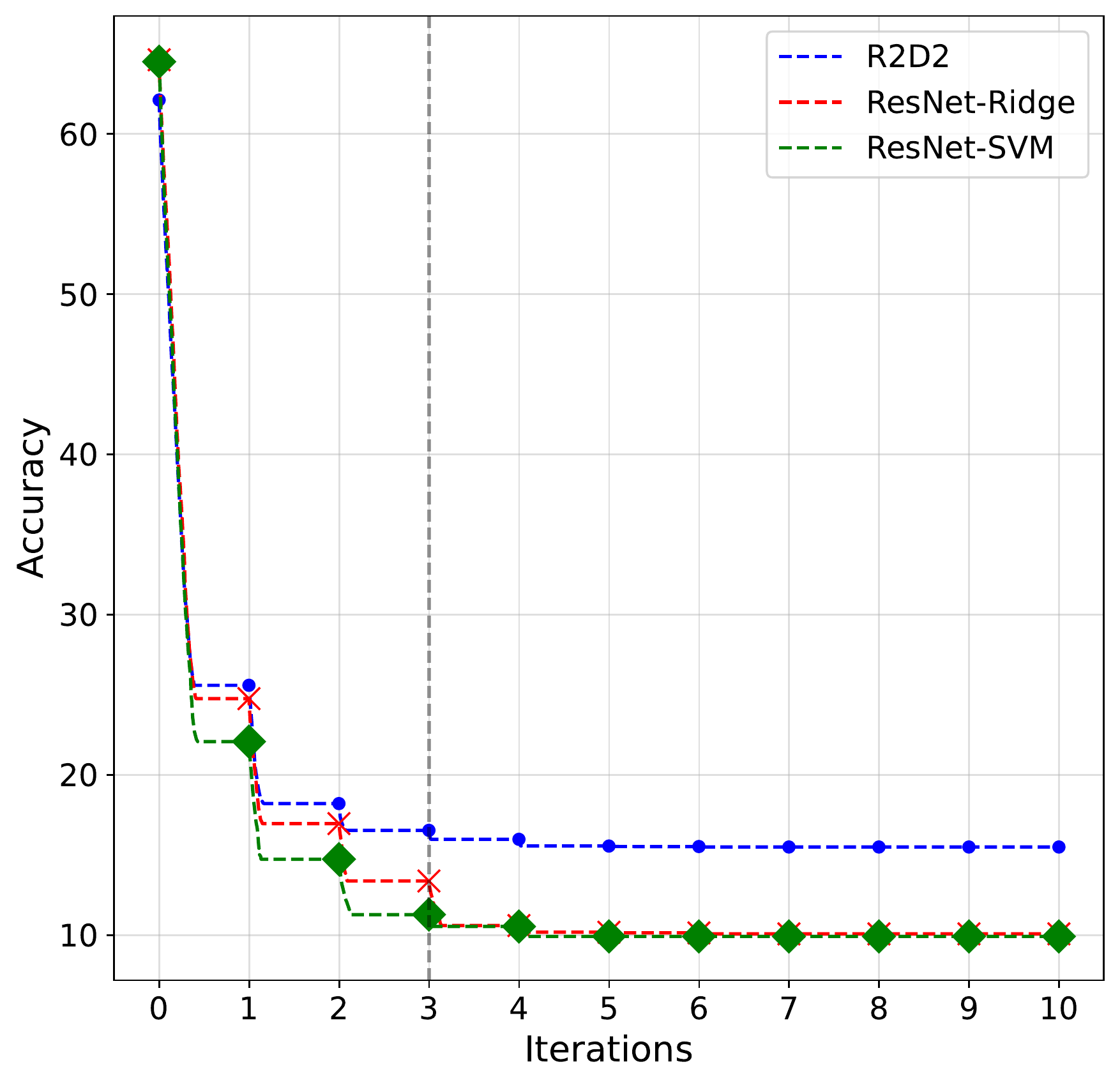}}
  \hfill
  
  \caption{Convergence plots for Algorithm \ref{alg:attack} for different meta-learning algorithms on the CIFAR-FS and FC100 datasets.}
  \label{fig:appendix:convergence}
\end{figure}

\section{Performance range results}
\label{sec:appendix:worstcase_examples}

\begin{figure}[t]
  \centering
  
  \subfigure[Examples of FC100 adaptation images from a random 1-shot test task and the corresponding post-adaptation accuracies of MetaOptNet-SVM. Accuracy ranges from 4.0\% to 60.8\% illustrating the sensitivity of the meta-learner to the adaptation data.]{\includegraphics[width=0.45\textwidth]{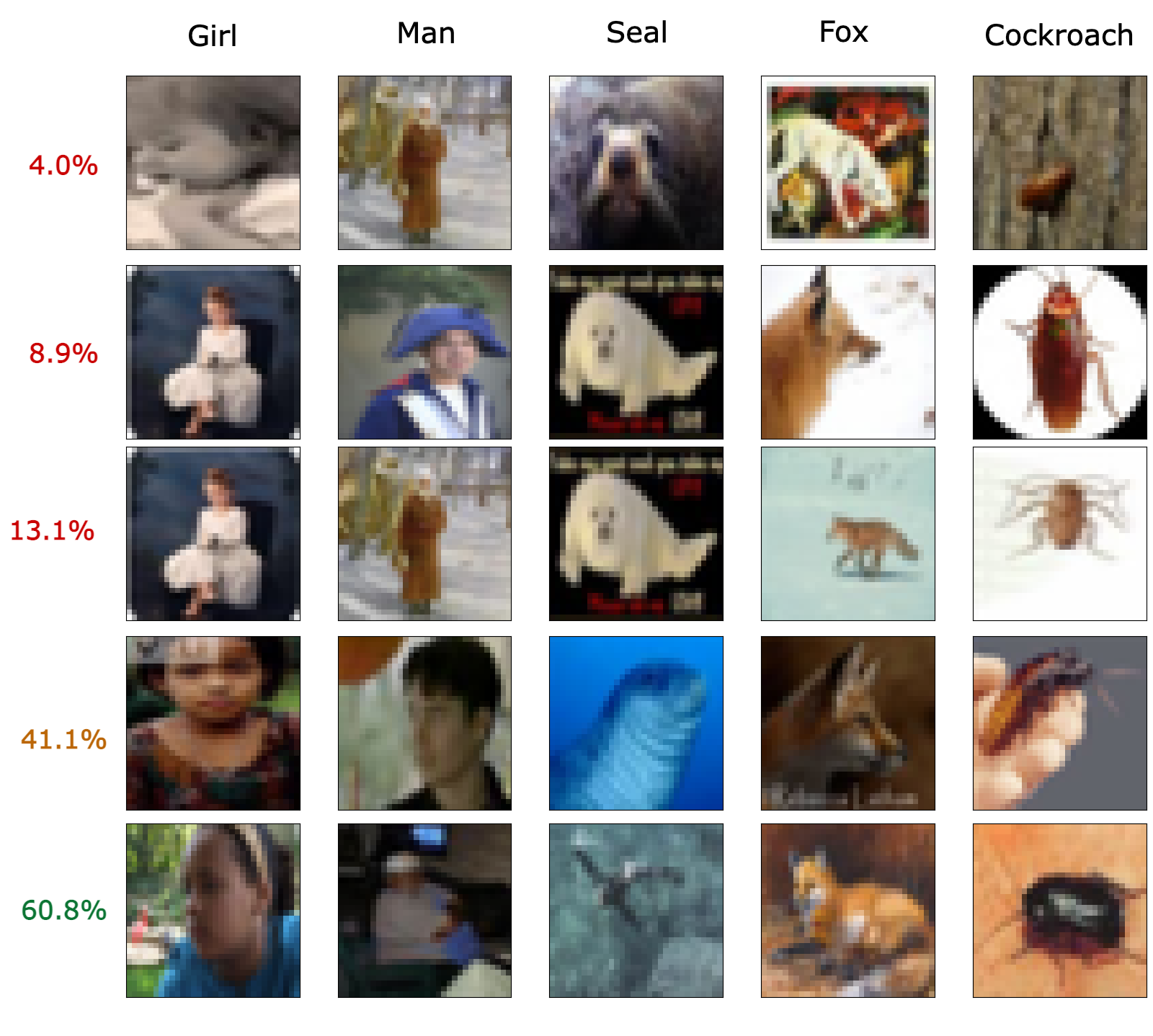}} 
  \hfill
  \subfigure[Examples of MiniImageNet adaptation images from a random 1-shot test task and the corresponding post-adaptation accuracies of MetaOptNet-SVM. Accuracy ranges from 5.8\% to 75.9\% illustrating the sensitivity of the meta-learner to the adaptation data.]{\includegraphics[width=0.45\textwidth]{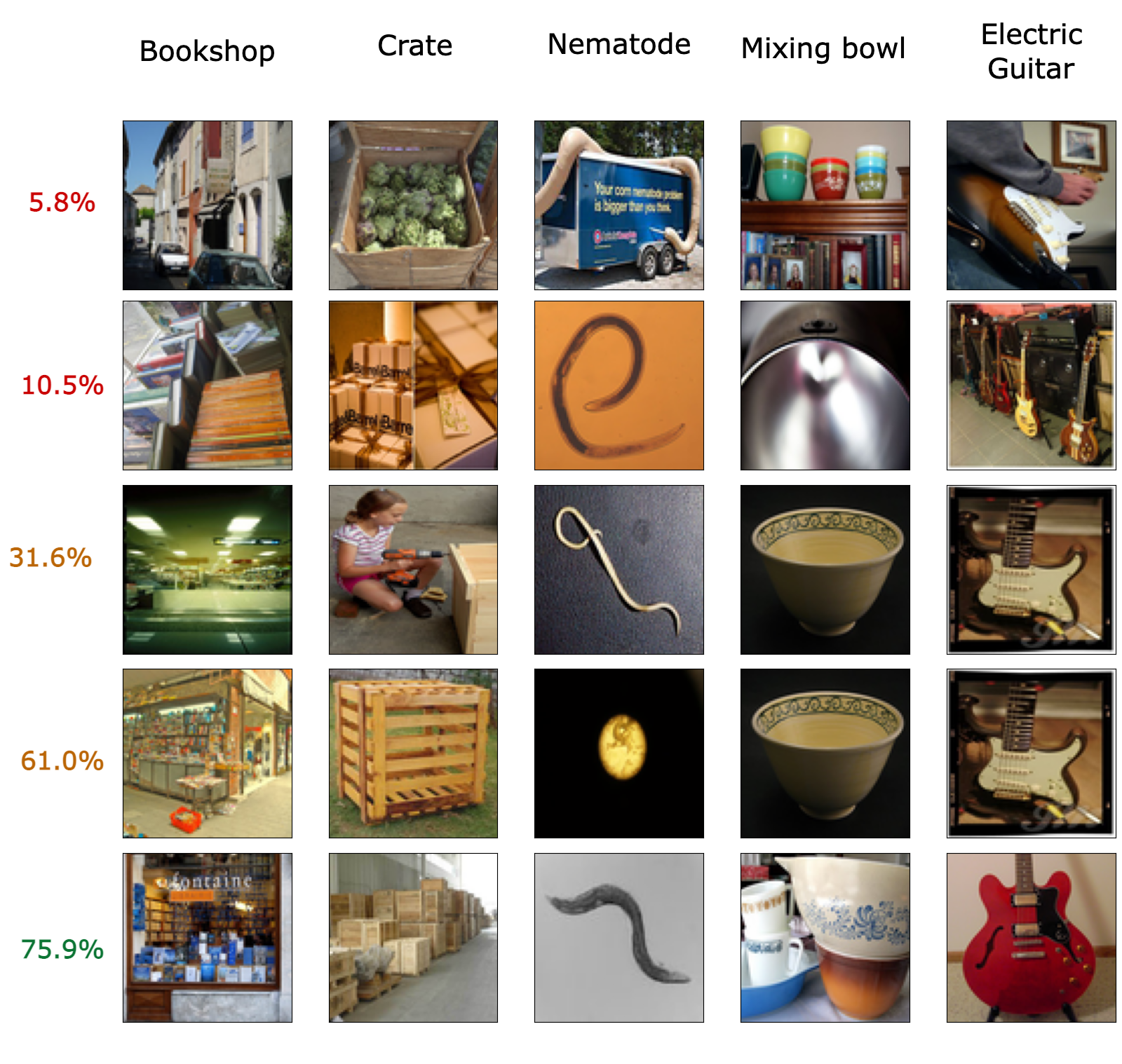}}\hfill
  
  \caption{Examples of unaltered support images from the FC100 and miniImageNet datasets for a random 1-shot task, depicting the post-adaptation performance of a popular meta-learning algorithm (MetaOptNet-SVM).}
  \label{fig:appendix:exemplars}
\end{figure}

\begin{figure}[]
  \centering
  
  \subfigure[Examples of CIFAR-FS adaptation images from a random 1-shot test task and the corresponding post-adaptation accuracies of ProtoNet.]{\includegraphics[width=0.32\textwidth]{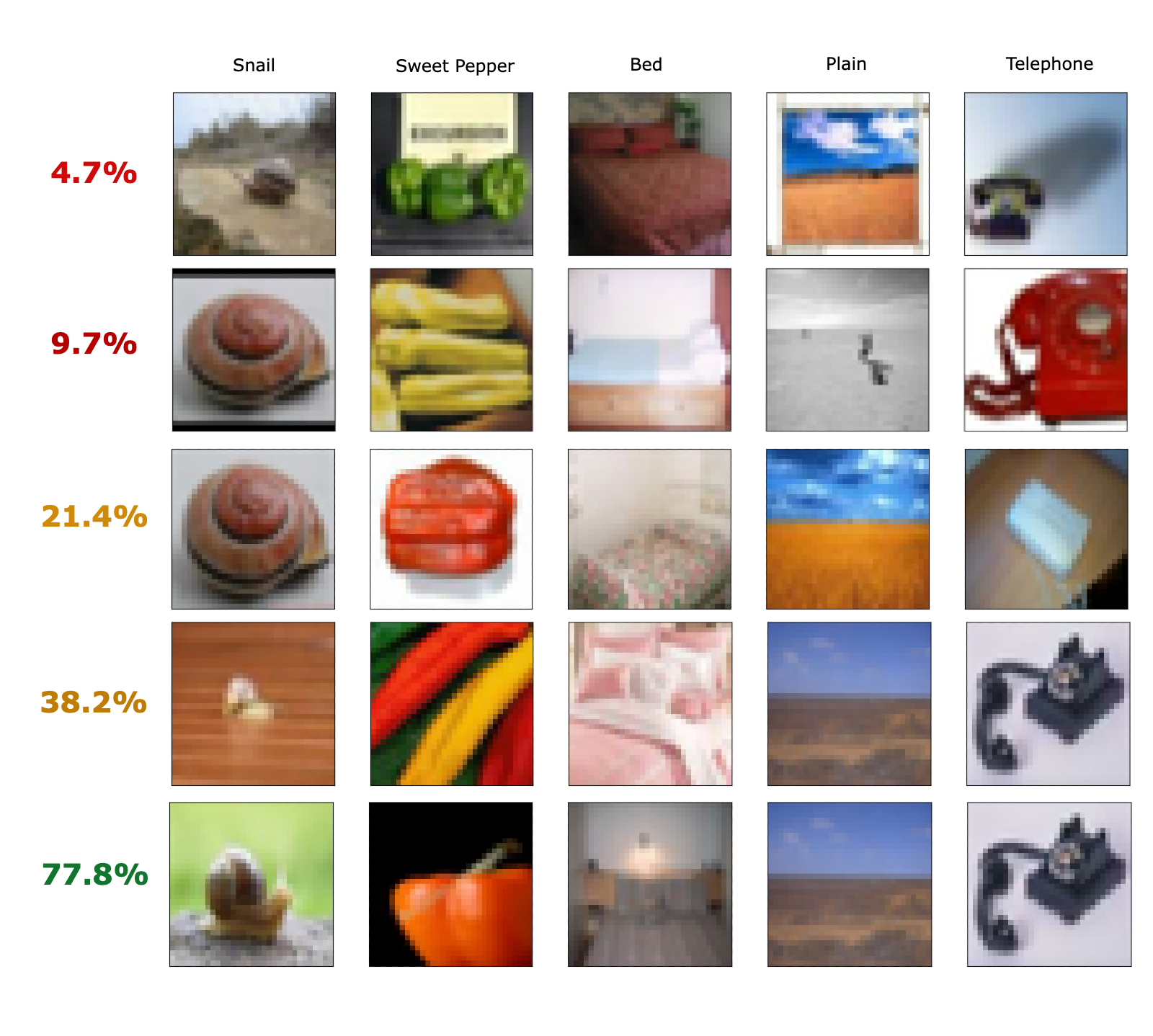}} 
  \hfill
  \subfigure[Examples of FC100 adaptation images from a random 1-shot test task and the corresponding post-adaptation accuracies of ProtoNet.]{\includegraphics[width=0.33\textwidth]{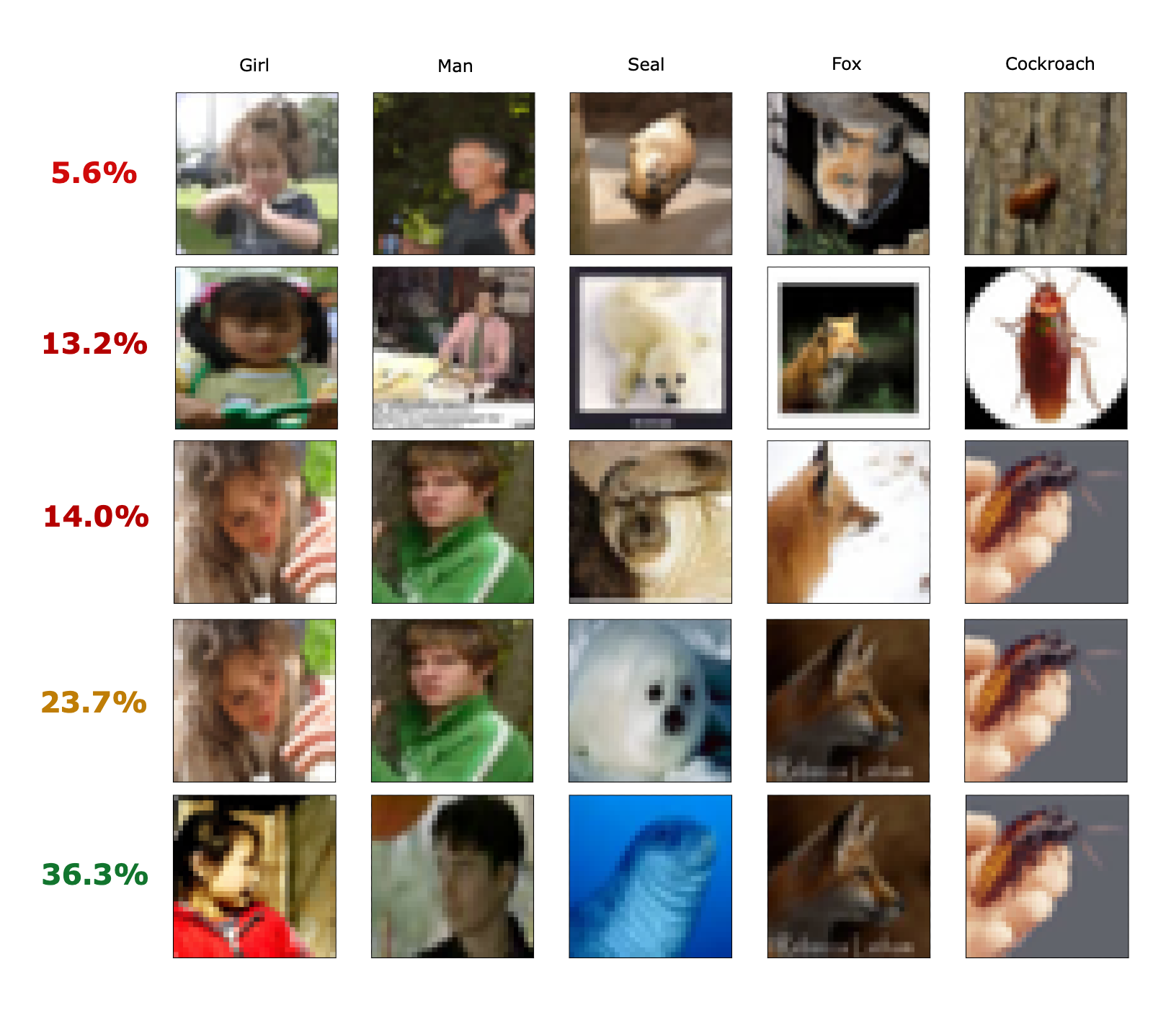}} 
  \hfill
  \subfigure[Examples of MiniImageNet adaptation images from a random 1-shot test task and the corresponding post-adaptation accuracies of ProtoNet.]{\includegraphics[width=0.33\textwidth]{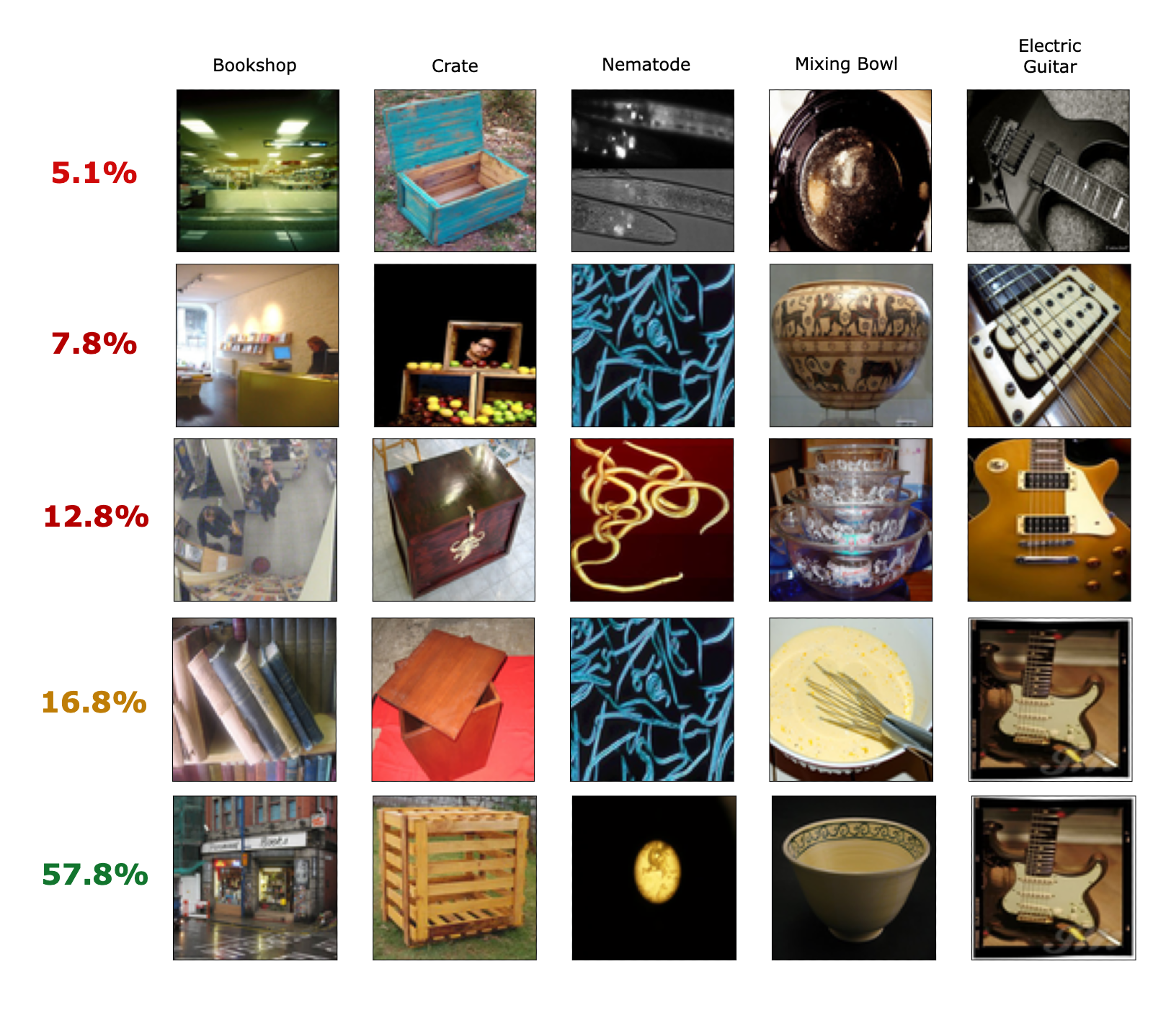}}\hfill
  
  \caption{Examples of unaltered support images from the CIFAR-FS, FC100, and miniImageNet datasets for a random 1-shot task, depicting the post-adaptation performance of a popular meta-learning algorithm (ProtoNets).}
  \label{fig:appendix:protonet:exemplars}
\end{figure}

\begin{enumerate}

    \item In Figure \ref{fig:appendix:exemplars} we present several examples of the worst-case support images on the FC100 and the miniImageNet datasets for the MetaOptNet-SVM method. Additionally, in Figure \ref{fig:appendix:protonet:exemplars} we present several examples of the worst-case support images on the CIFAR-FS, FC100, and the miniImageNet datasets for the ProtoNets method.
    All the support images are correctly labeled and appear representative of the respective classes. Closely inspecting the images, we note that it is often not easy to notice visually that such support examples could result in a poor performance without significant expert knowledge of the dataset. Barring a very small portion of the images (e.g., gray-scale ``Girl'' image and a truck-size ``Nematode''), images appear reasonable and adequately representative of their respective classes.
    
    \item In Figure \ref{fig:appendix:attackprog} we present histograms of accuracies visualizing the first iteration over classes of Algorithm \ref{alg:attack} in 1-shot learning on miniImageNet dataset for MAML, R2D2, MetaOptNet-Ridge, and the MetaOptNet-SVM algorithms. The rightmost histogram (in each sub-figure) corresponds to post-adaptation accuracies for different choices of support image for class 0 and random choices for classes 1-4. The subsequent histogram (in each sub-figure) is for different choices of support images for class 1, where image for class 0 is chosen with Algorithm \ref{alg:attack} and classes 2-4 are random, and analogously for the remaining three histograms. We see that the lower accuracy tails of the first two histograms contain multiple worst-case support examples in the corresponding classes. By the third histogram, the range of accuracies is well below the average accuracies for each of the algorithms for \emph{all} possible support examples in the corresponding classes.
    
    \item In Figure \ref{fig:appendix:accshist} we present histogram of accuracies of unique combinations from the miniImageNet dataset of 1-shot support examples evaluated by Algorithm \ref{alg:attack} throughout 3 iterations. 
    
    \begin{enumerate}
        \item For the MAML algorithm, out of the 4390 distinct sets represented in the histogram, 3054 distinct sets of 5 examples each have less than 30\% post-adaptation accuracy.
        \item For the R2D2 algorithm, out of the 5986 distinct sets represented in the histogram, 5274 distinct sets of 5 examples each have less than 40\% post-adaptation accuracy.
        \item For the MetaOptNet-Ridge algorithm, out of the 3592 distinct sets represented in the histogram, 2795 distinct sets of 5 examples each have less than 40\% post-adaptation accuracy.
        \item For the MetaOptNet-SVM algorithm, out of the 5188 distinct sets represented in the histogram, 4403 distinct sets of 5 examples each have less than 40\% post-adaptation accuracy.
    \end{enumerate}
\end{enumerate}

\begin{figure}[t]
  \centering
  
  \subfigure[MAML]{\includegraphics[width=0.24\textwidth]{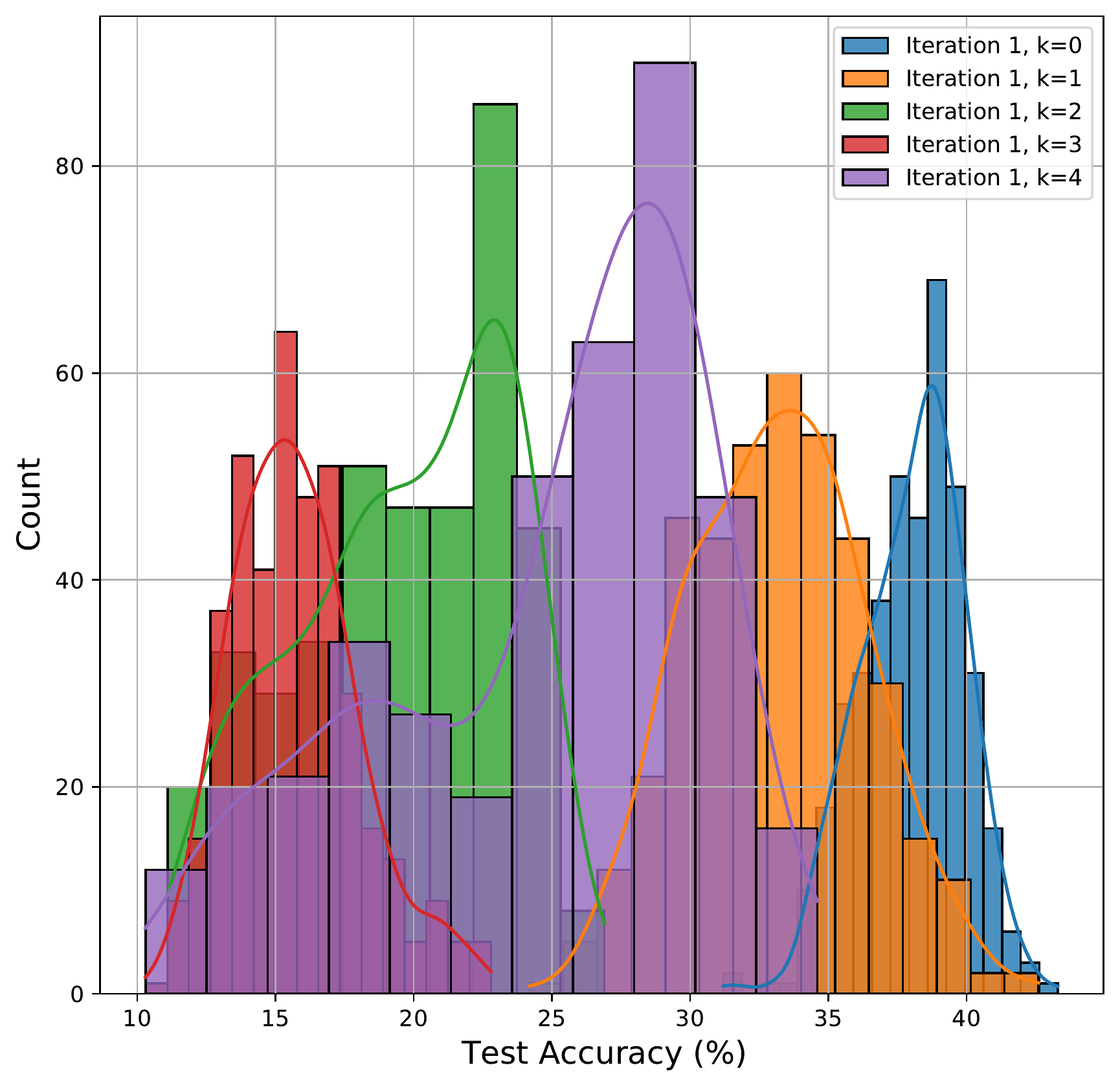}} \hfill
  \subfigure[R2D2]{\includegraphics[width=0.24\textwidth]{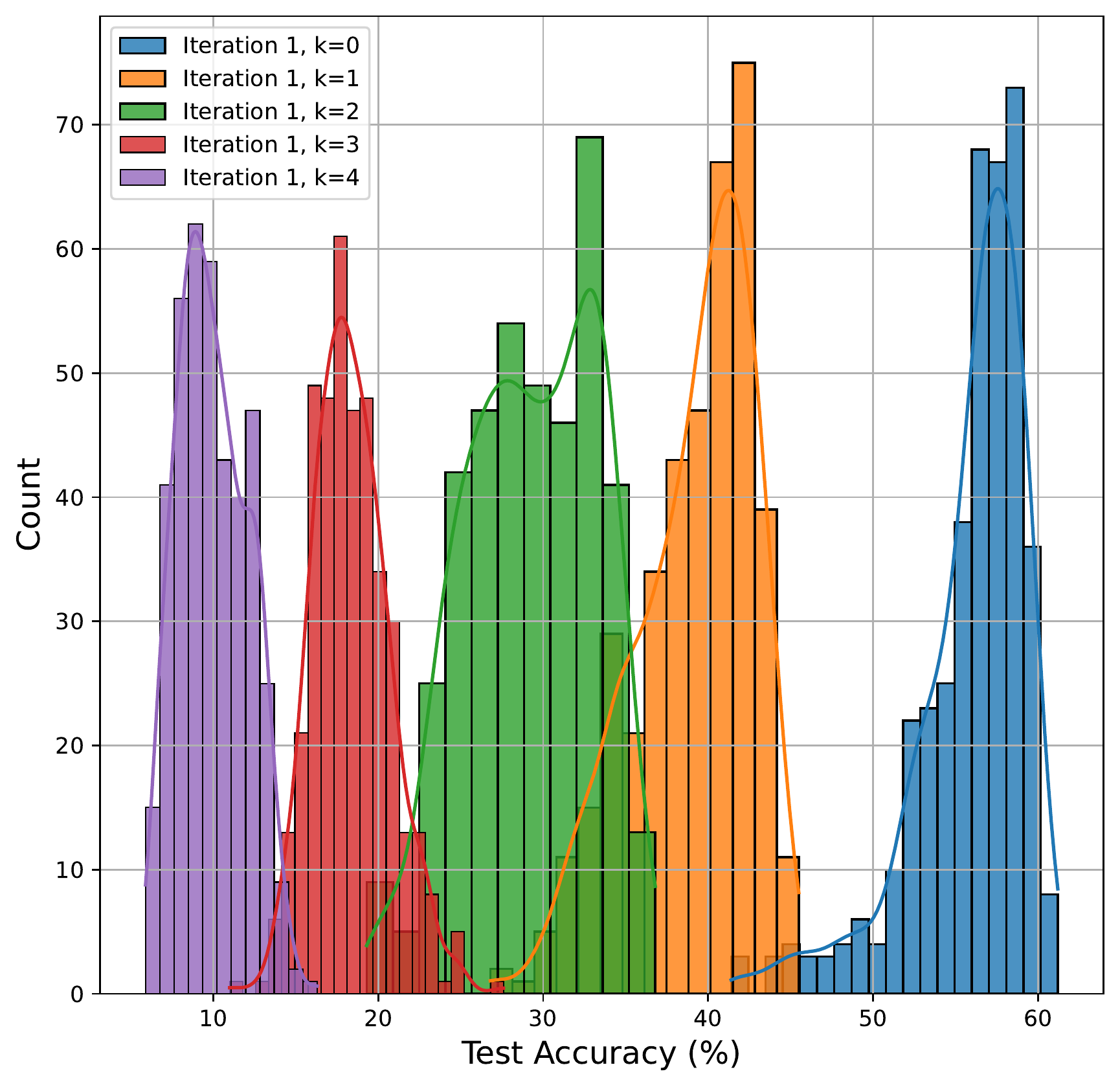}}\hfill
  \subfigure[MetaOptNet-Ridge]{\includegraphics[width=0.24\textwidth]{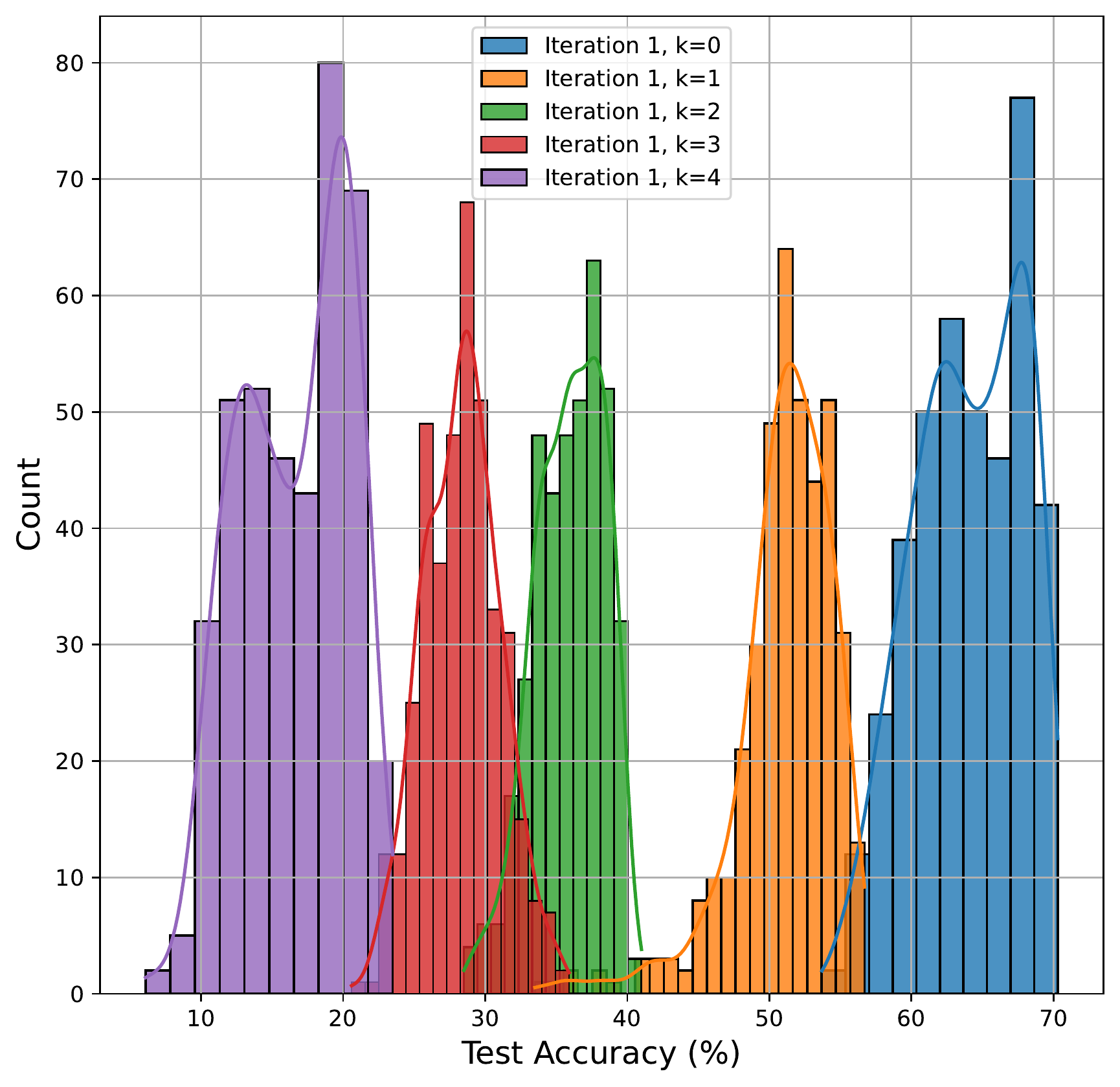}} \hfill
  \subfigure[MetaOptNet-SVM]{\includegraphics[width=0.24\textwidth]{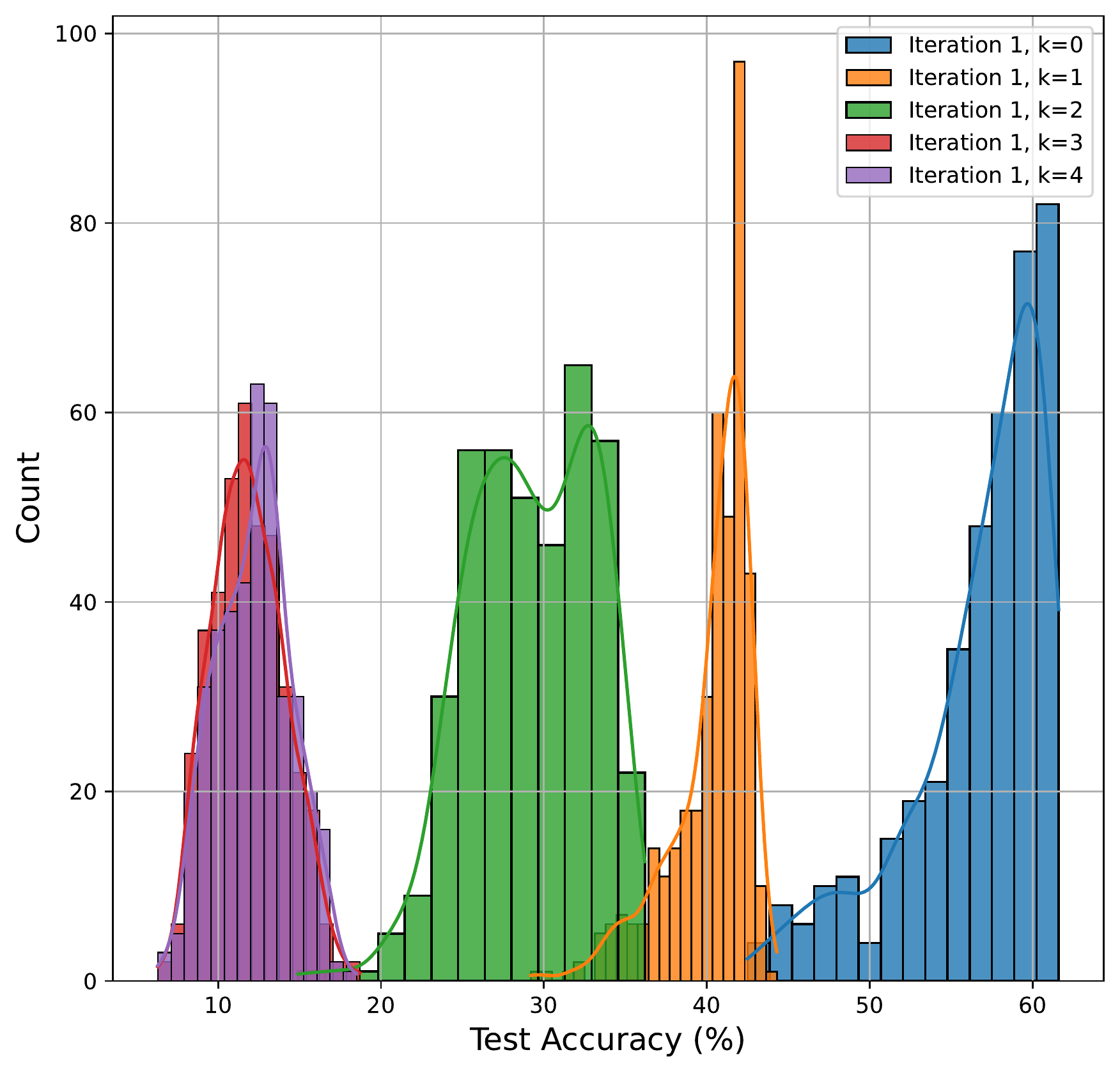}}\hfill
  
  \caption{Histogram of accuracies visualizing progression of the first iteration of Algorithm \ref{alg:attack} in 1-shot learning over the miniImageNet dataset for different meta-learning algorithms.}
  \label{fig:appendix:attackprog}
\end{figure}

\begin{figure}[]
  \centering
  
  \subfigure[MAML]{\includegraphics[width=0.24\textwidth]{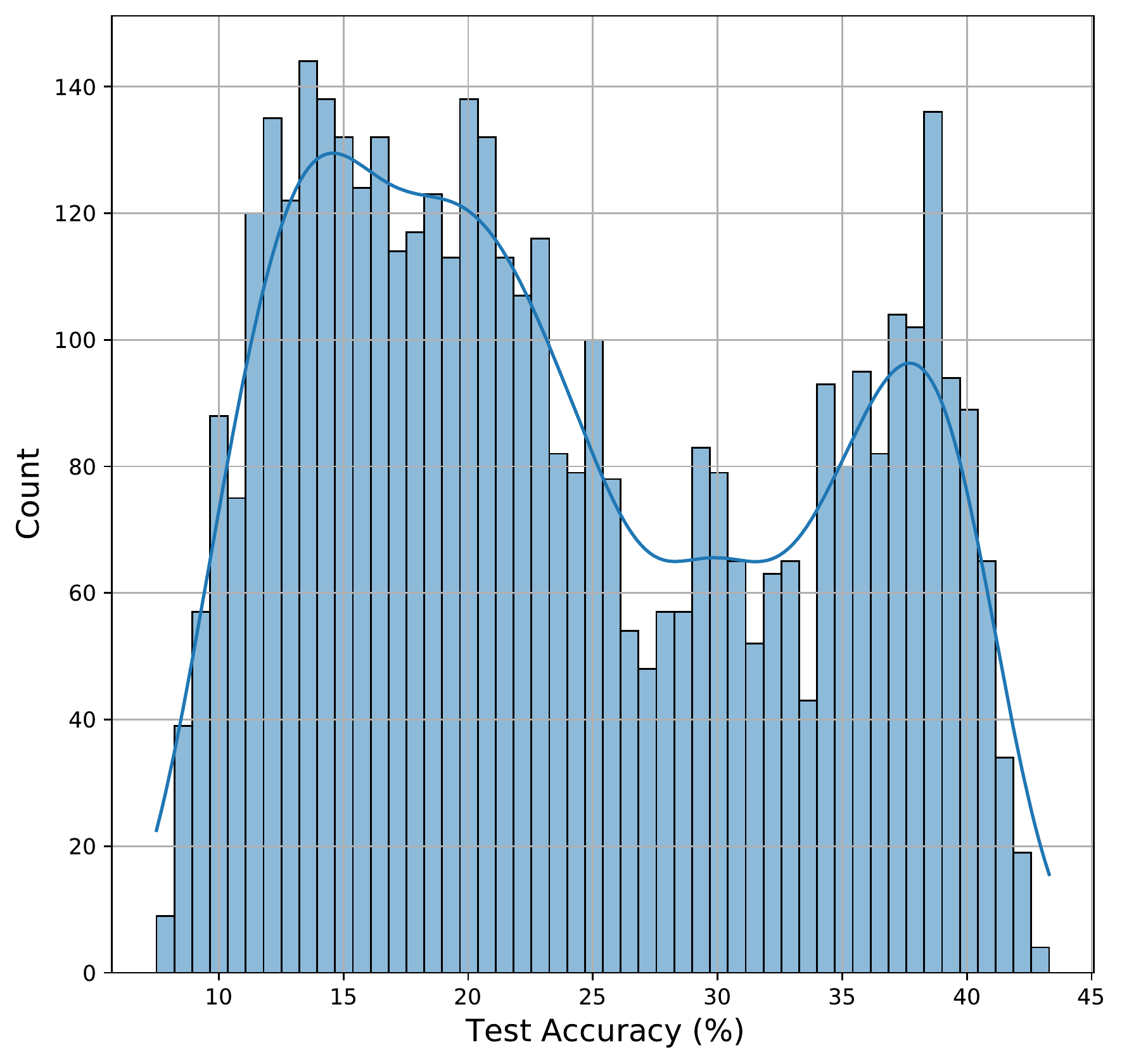}} \hfill
  \subfigure[R2D2]{\includegraphics[width=0.24\textwidth]{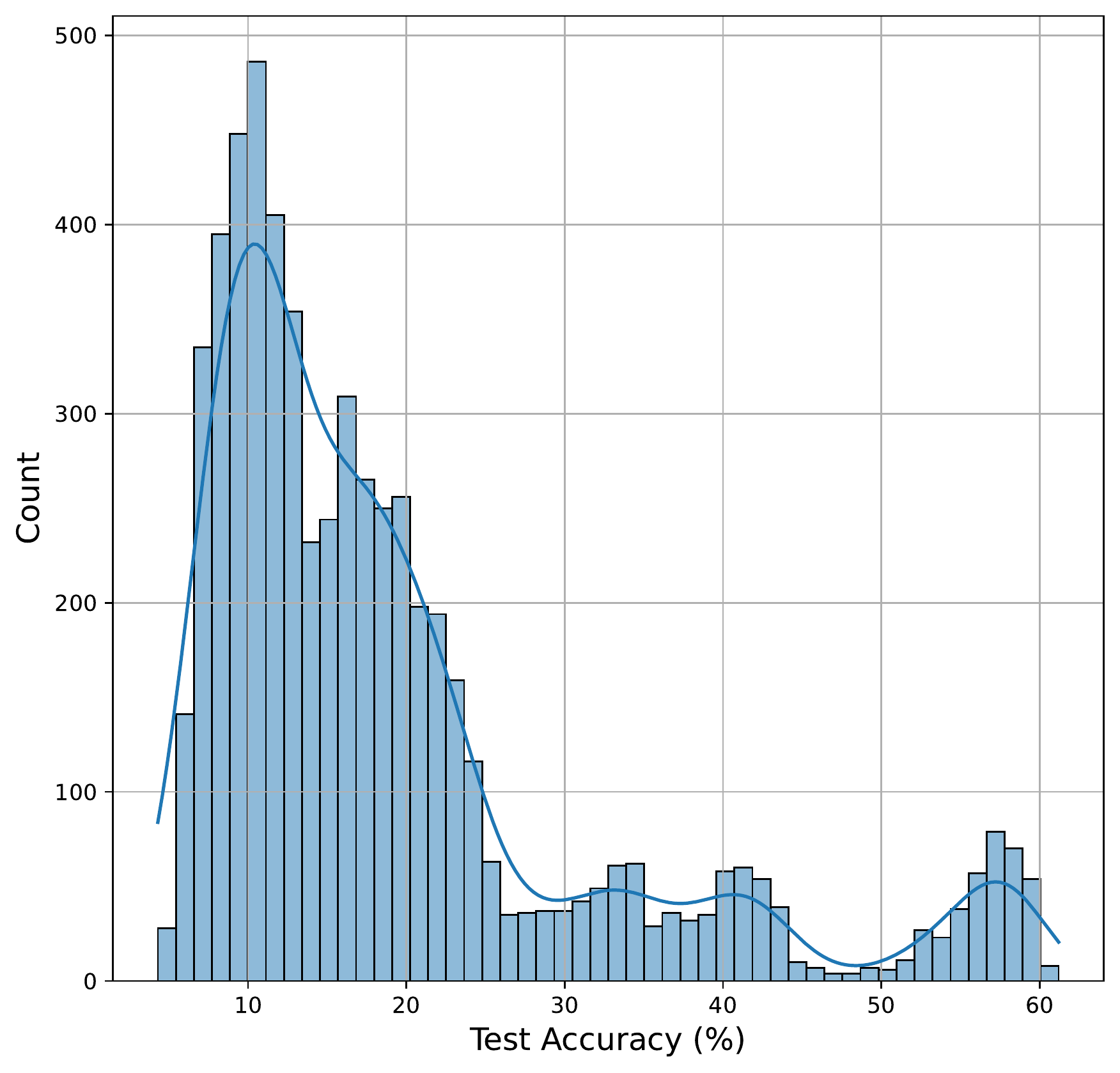}}\hfill
  \subfigure[MetaOptNet-Ridge]{\includegraphics[width=0.24\textwidth]{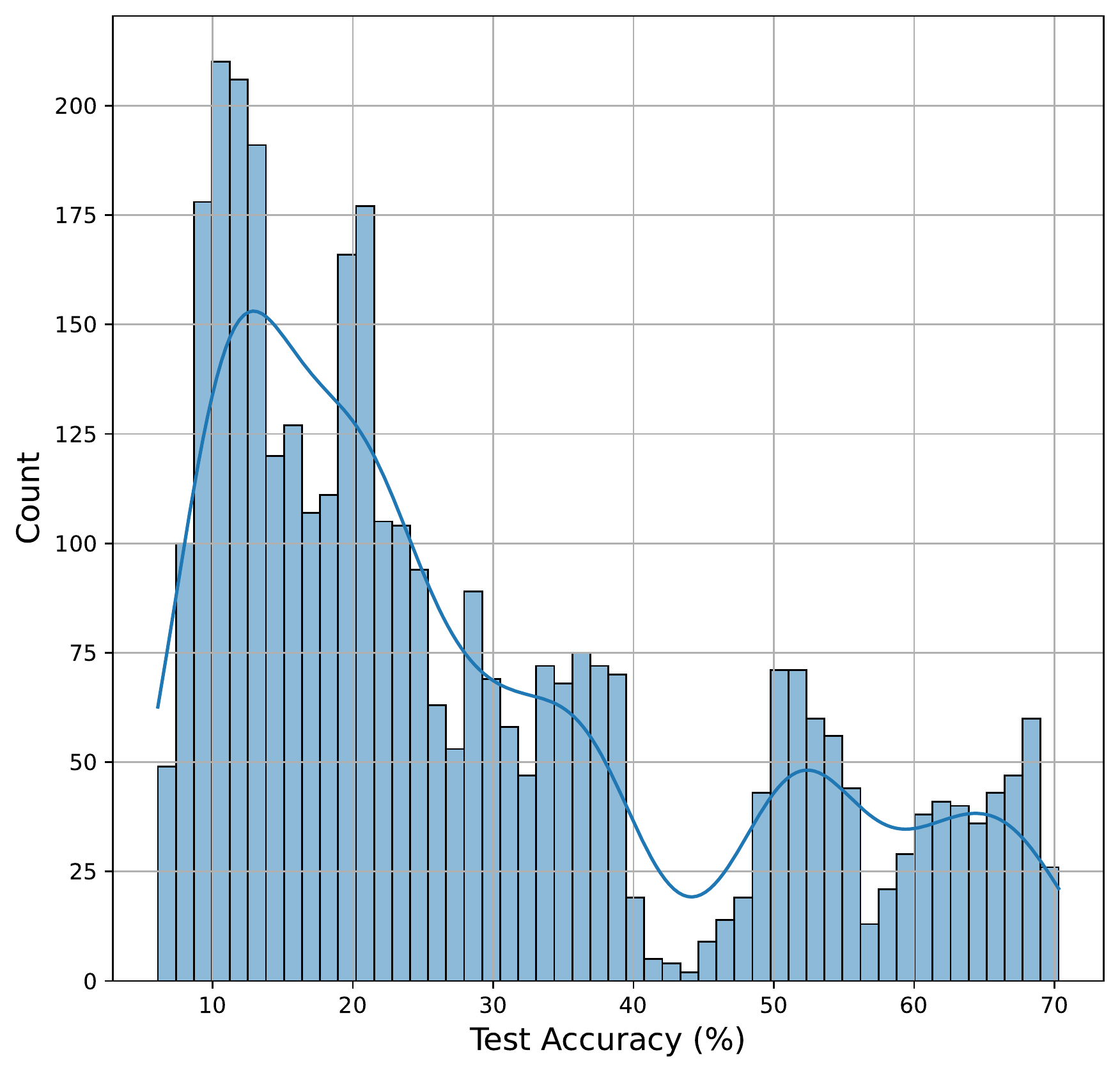}} \hfill
  \subfigure[MetaOptNet-SVM]{\includegraphics[width=0.24\textwidth]{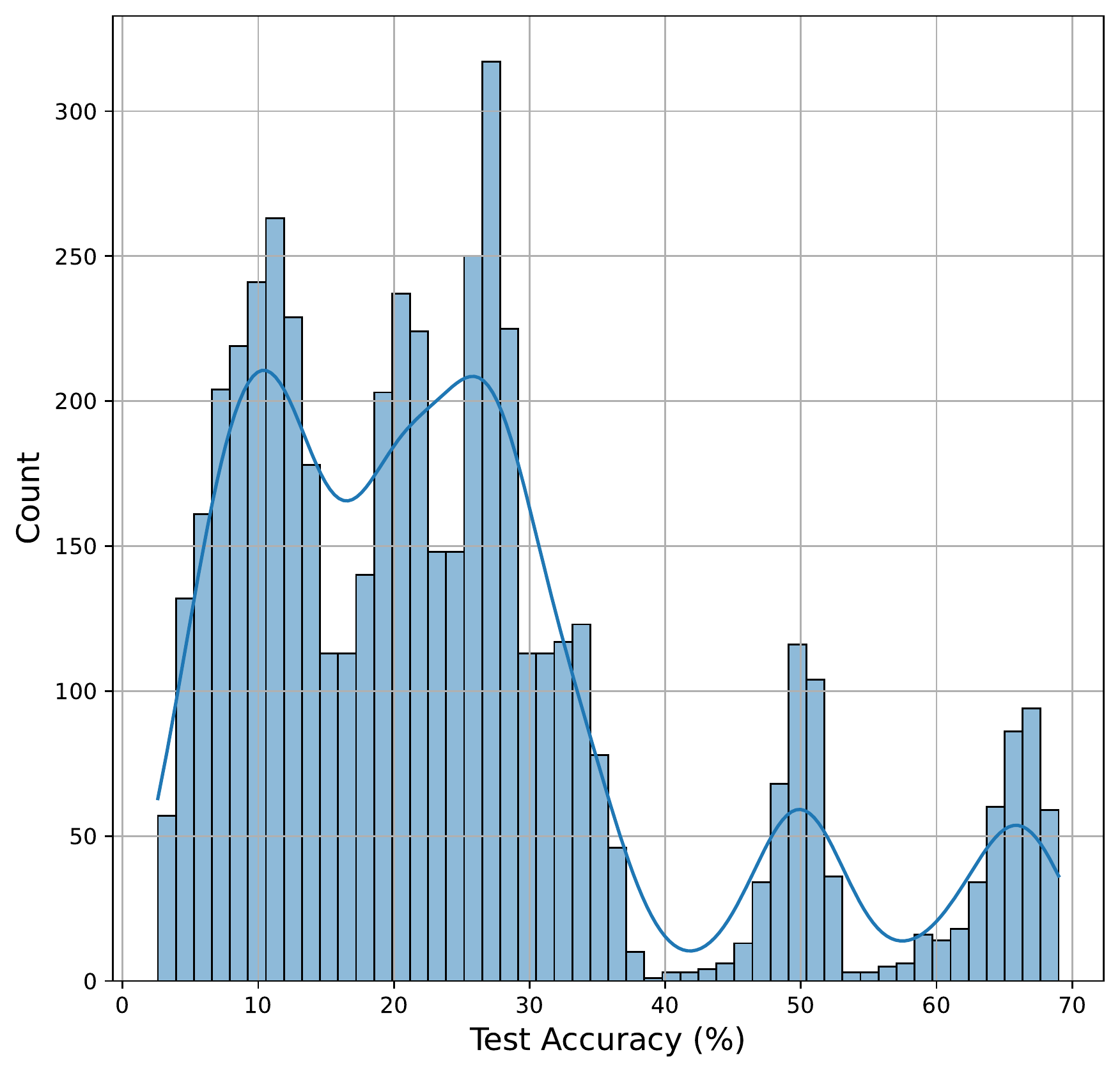}}\hfill
  
  \caption{Histogram of accuracies of unique combinations from the miniImageNet dataset of 1-shot support examples evaluated by Algorithm \ref{alg:attack} for 3 iterations,}
  \label{fig:appendix:accshist}
\end{figure}

\section{Improving support data robustness with adversarial training}
\label{sec:appendix:embeddings}

In Table \ref{tab:appendix:embeddings}, we show the projected embeddings for the R2D2, MetaOptNet-Ridge, and the MetaOptNet-SVM algorithms on the training and the test dataset, when trained in a standard manner vs when trained adversarially. As we note in Section \ref{sec:advtraining} and as results depict in Tables  \ref{tab:cifarfs-adversarial} and \ref{tab:fc100-adversarial}, the adversarial training converges and the worst-case accuracy improves drastically on the training tasks while no improvement is observed on the test tasks.

\begin{table}[h]
    \centering
    
    \caption{Projected embeddings of R2D2, MetaOptNet-Ridge, and MetaOptNet-SVM methods for a train and a test task query data from the FC100 dataset. We compare standard training and adversarial training discussed in Section \ref{sec:advtraining}. Points are colored with their labels. Highlighted points are the worst-case support examples selected with Algorithm \ref{alg:attack}.}
    \label{tab:appendix:embeddings}
    
    \begin{tabular}{|c|cccc|}
    \toprule
        Training ($\downarrow$)        &      Dataset ($\downarrow$)     &       R2D2        &       MetaOptNet-Ridge        &       MetaOptNet-SVM      \\
    \toprule
    
    \multirow{2}{*}{Standard}    &   Train  &
                                \includegraphics[width=0.2\textwidth]{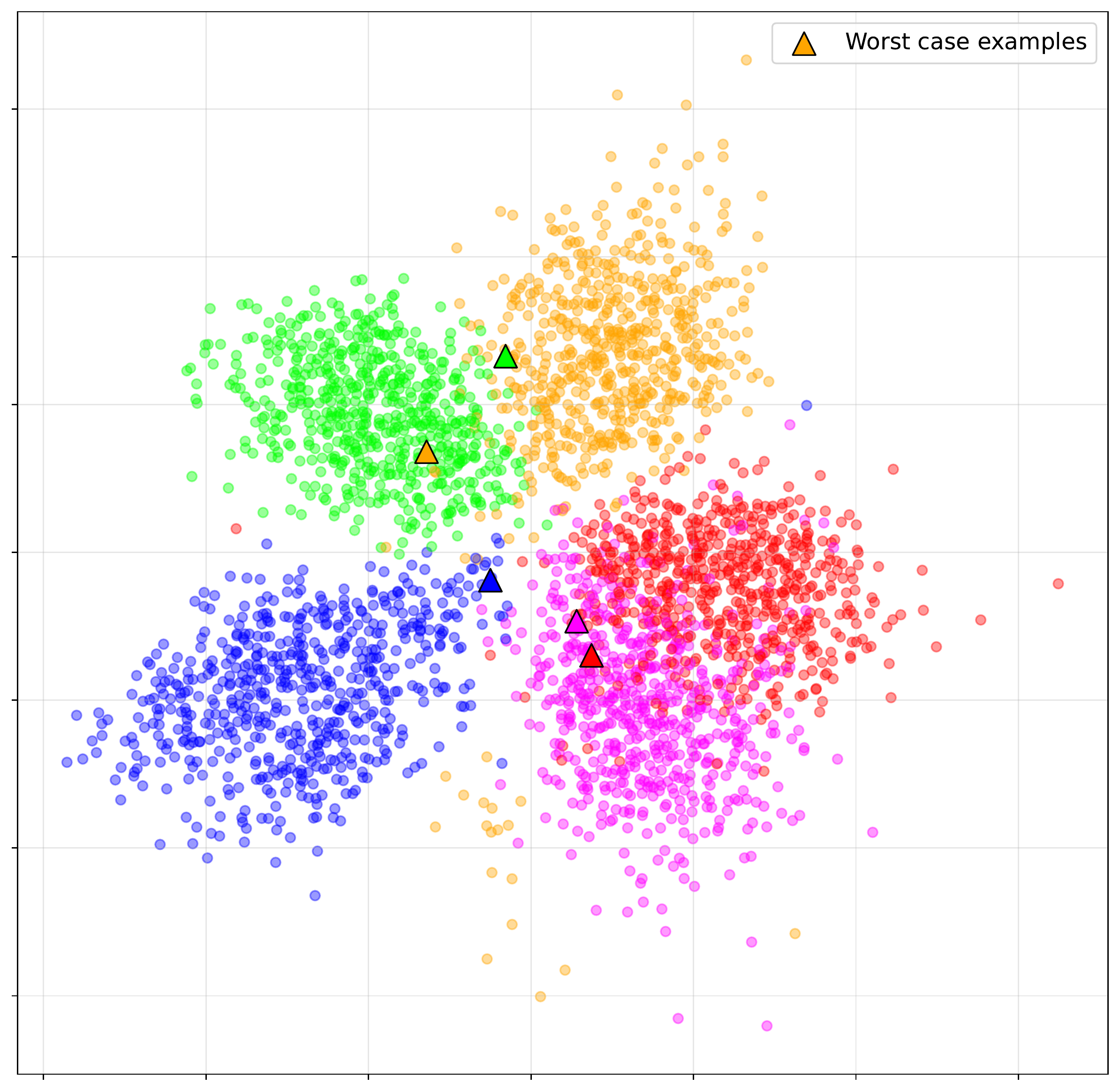}
                                            &       
                                \includegraphics[width=0.2\textwidth]{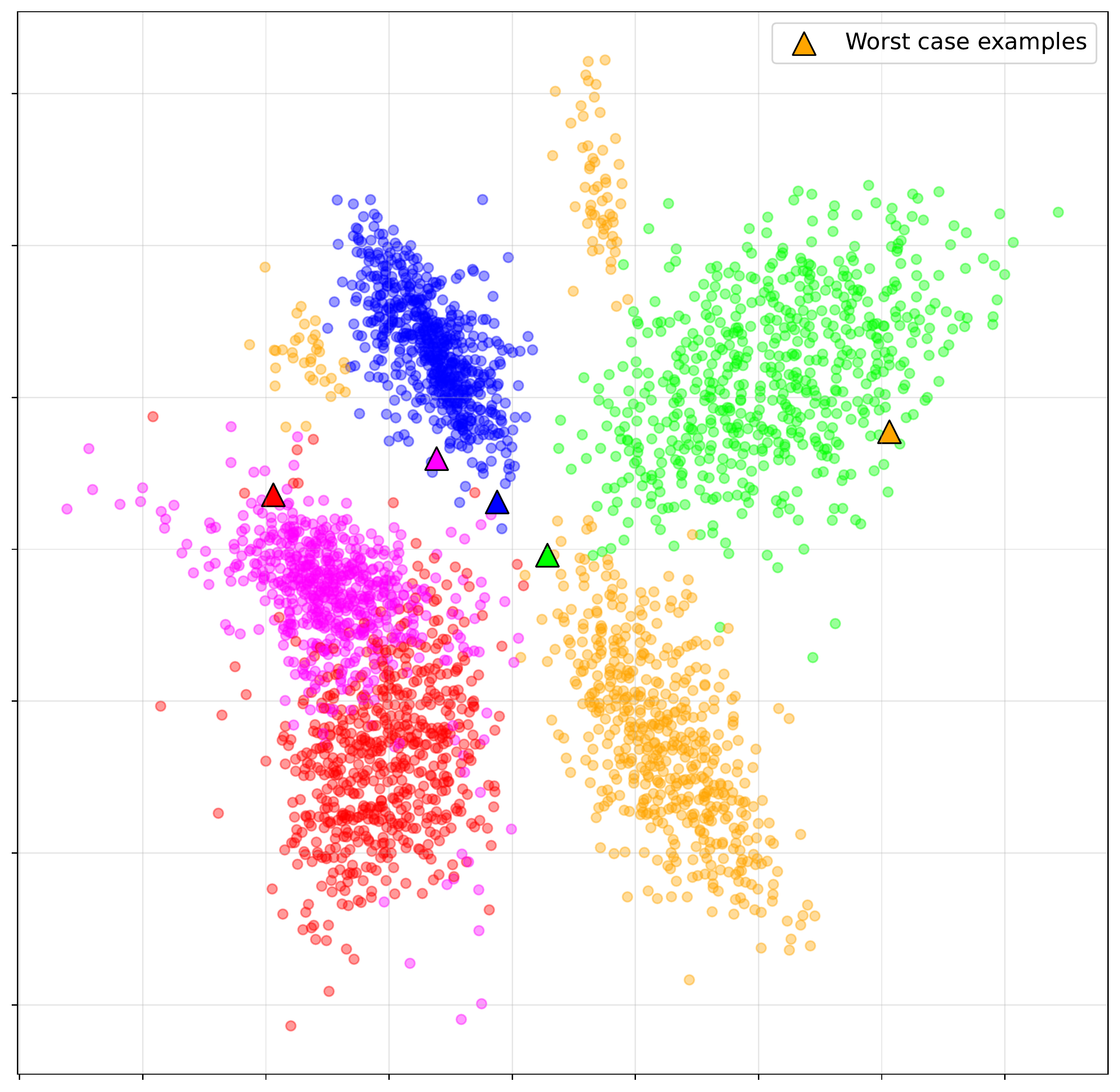}
                                            &   
                                \includegraphics[width=0.2\textwidth]{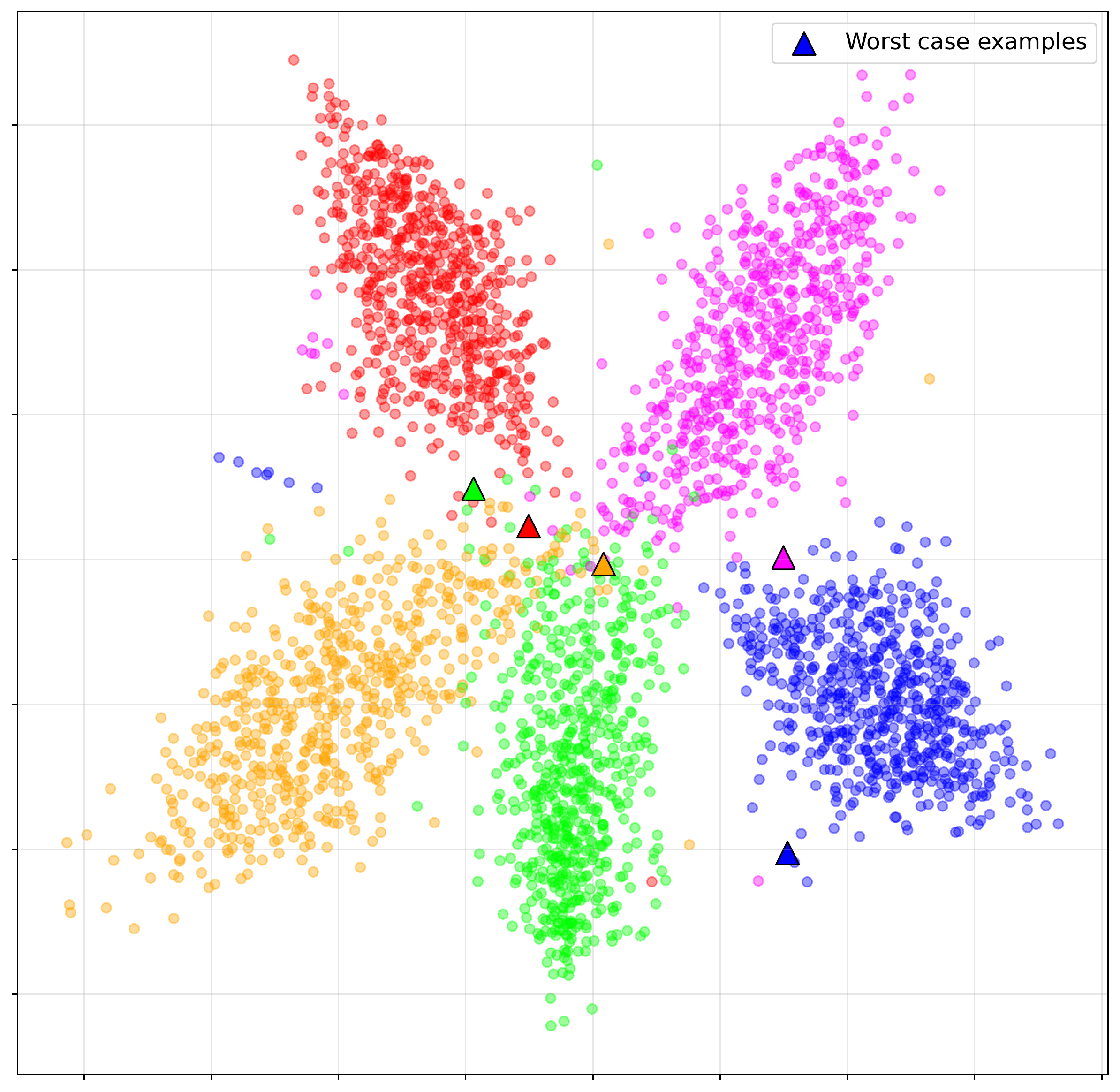}
                                \\
                    \cmidrule(r){2-5}
                                &   Test    &       
                                \includegraphics[width=0.2\textwidth]{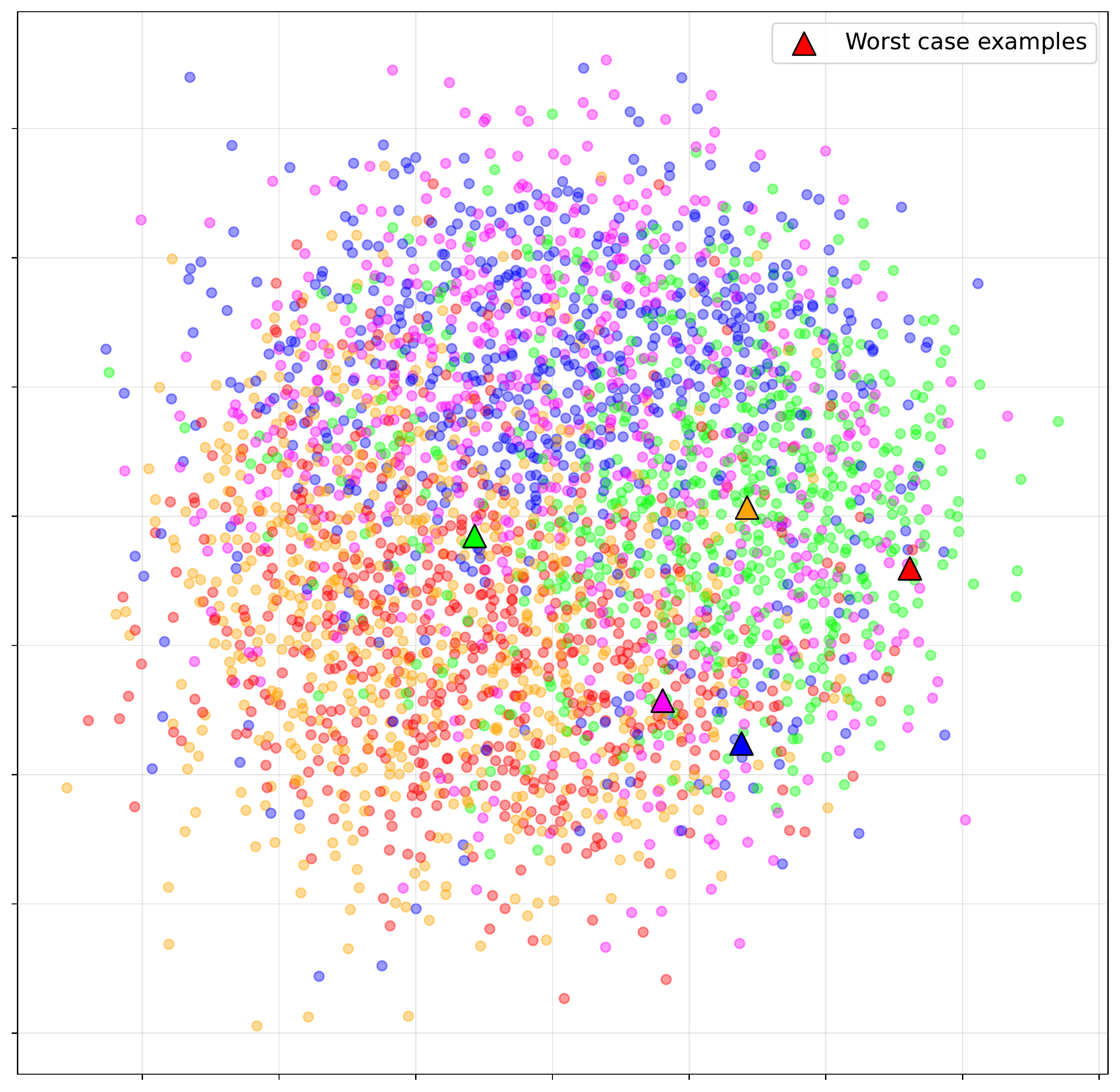}
                                            &       
                                \includegraphics[width=0.2\textwidth]{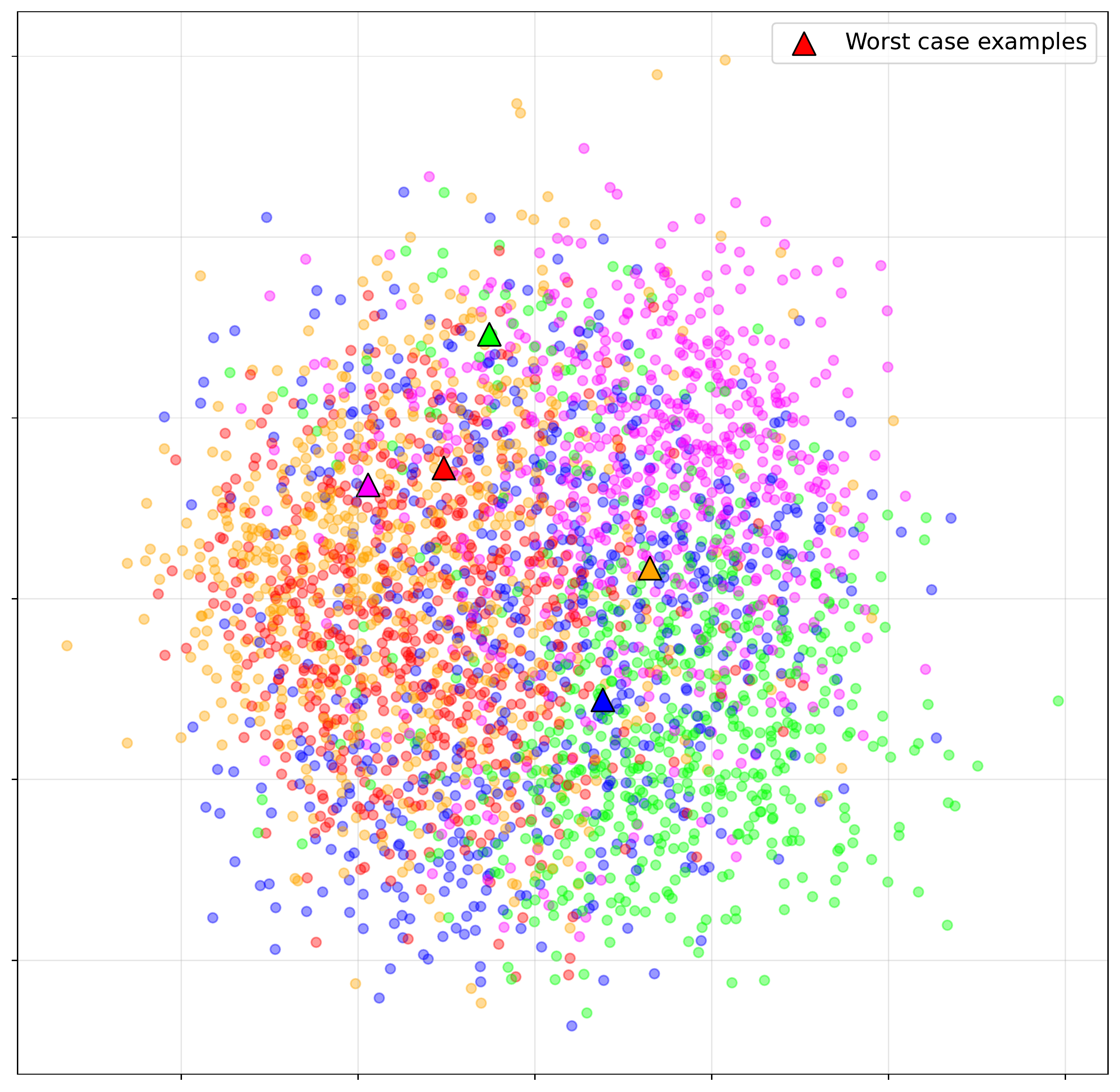}
                                            &           
                                \includegraphics[width=0.2\textwidth]{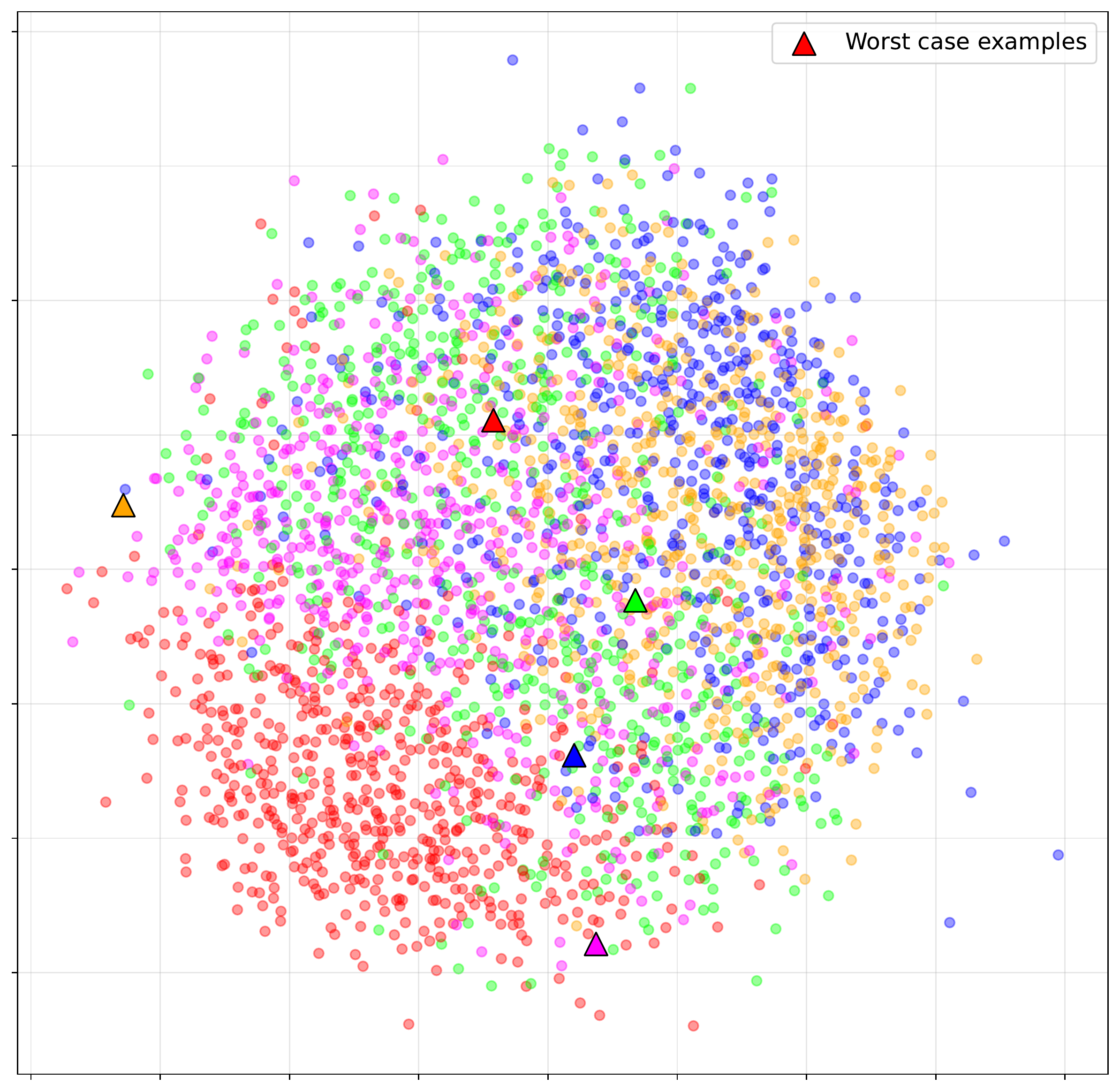}
                                \\
                                        
    \midrule         
    \multirow{2}{*}{Adversarial}    &   Train   &
                                \includegraphics[width=0.2\textwidth]{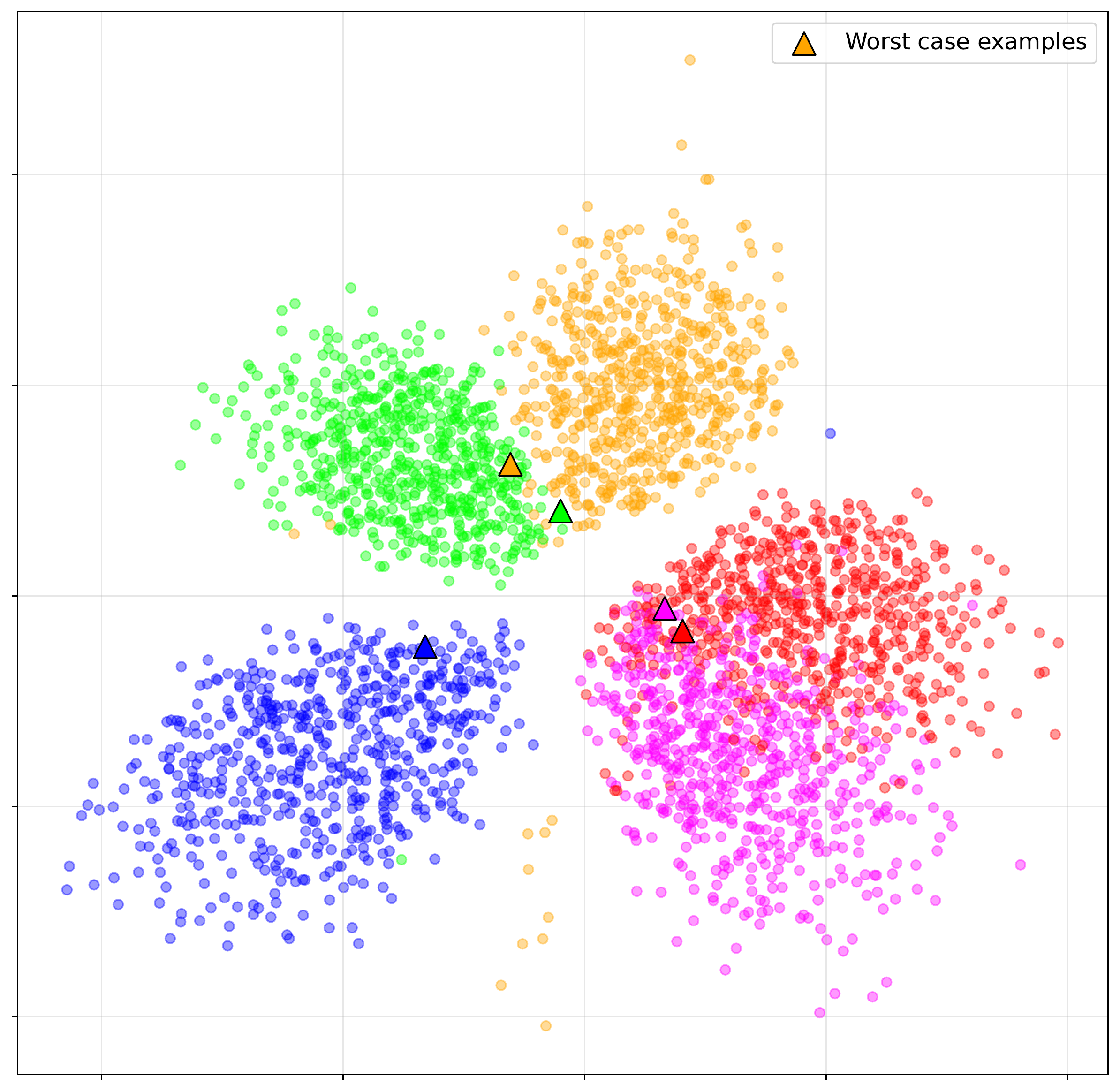}
                                            &       
                                \includegraphics[width=0.2\textwidth]{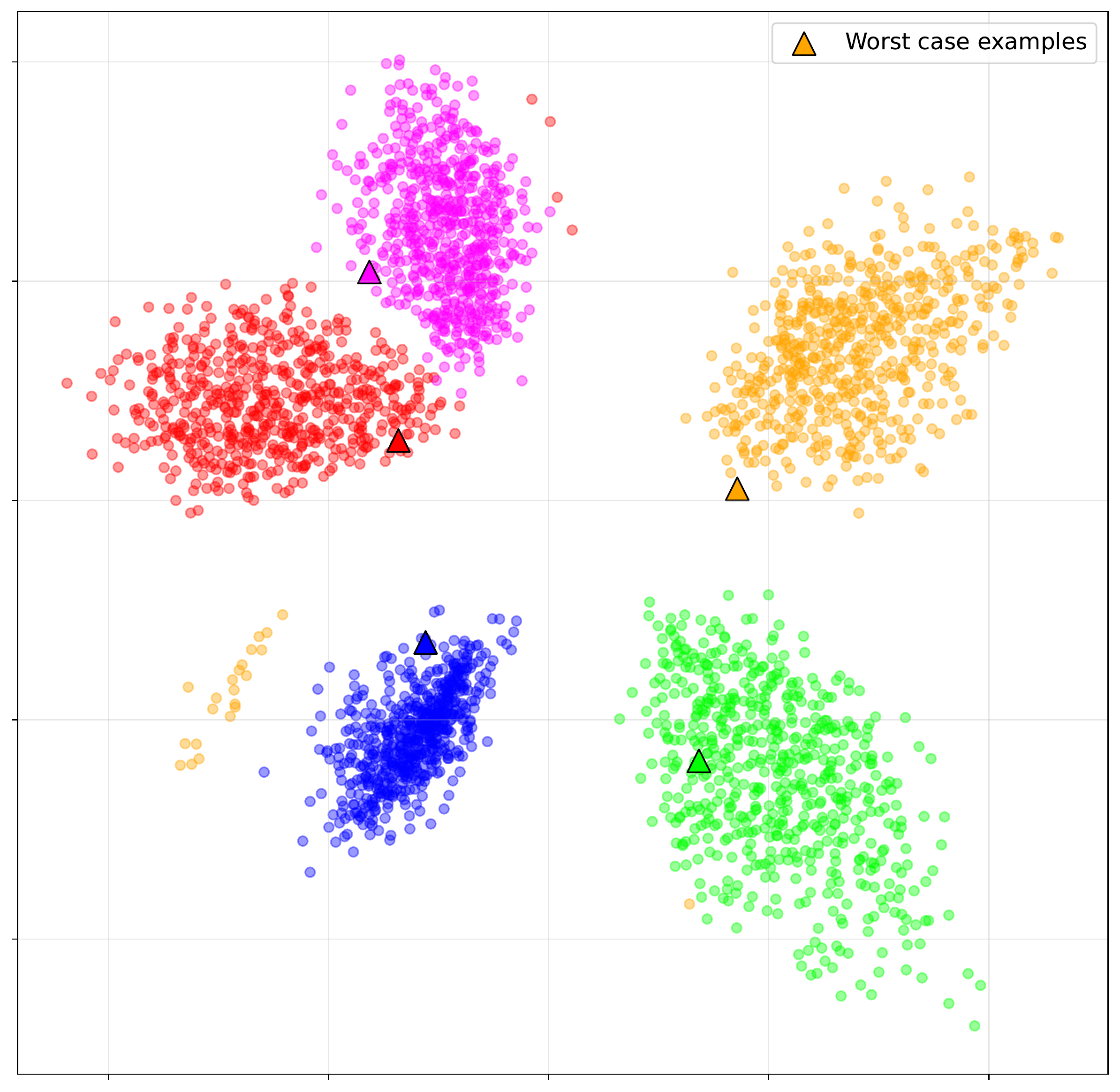}
                                            &   
                                \includegraphics[width=0.2\textwidth]{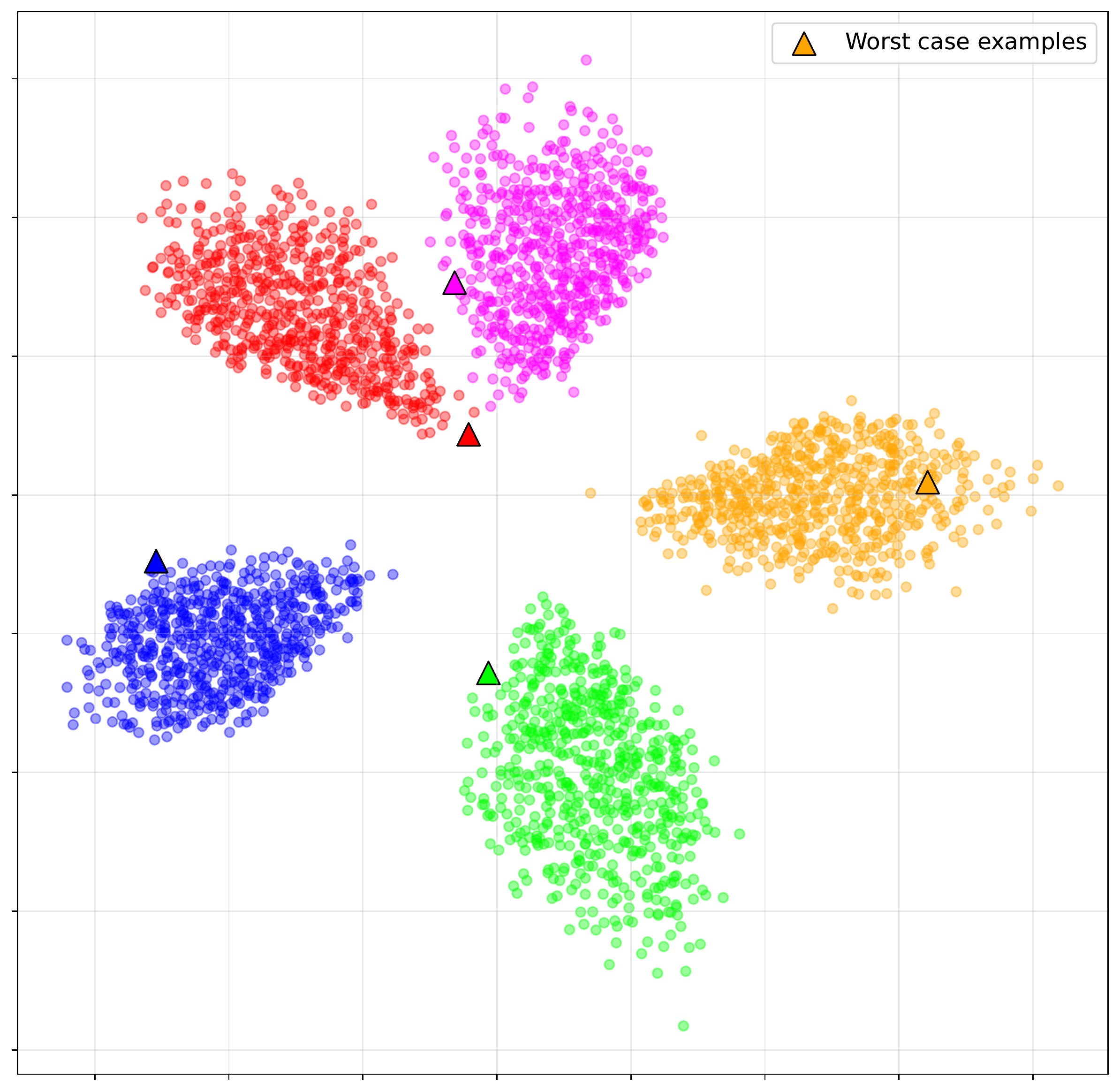}
                                \\
                    \cmidrule(r){2-5}
                                &   Test    &       
                                \includegraphics[width=0.2\textwidth]{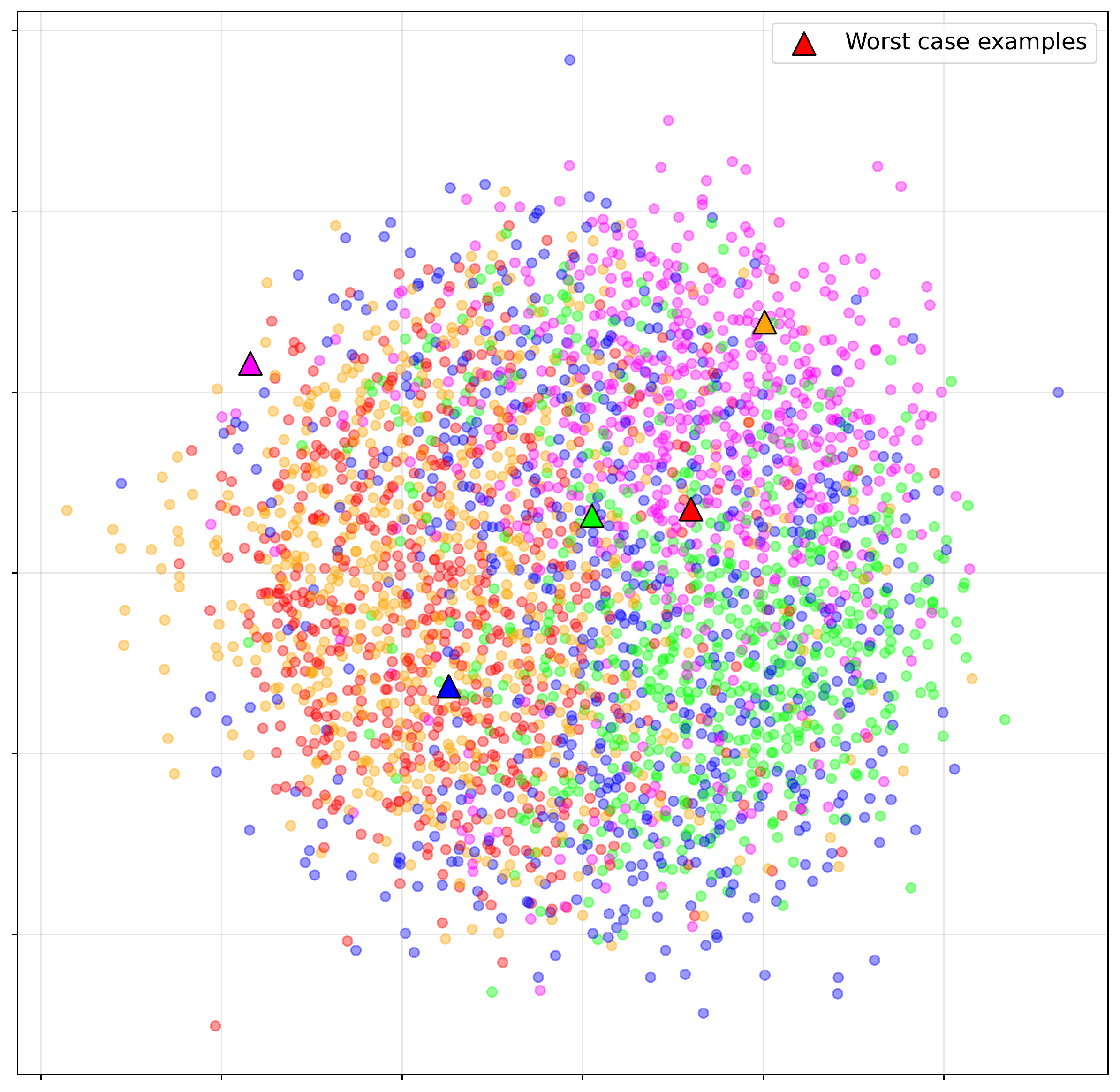}
                                            &       
                                \includegraphics[width=0.2\textwidth]{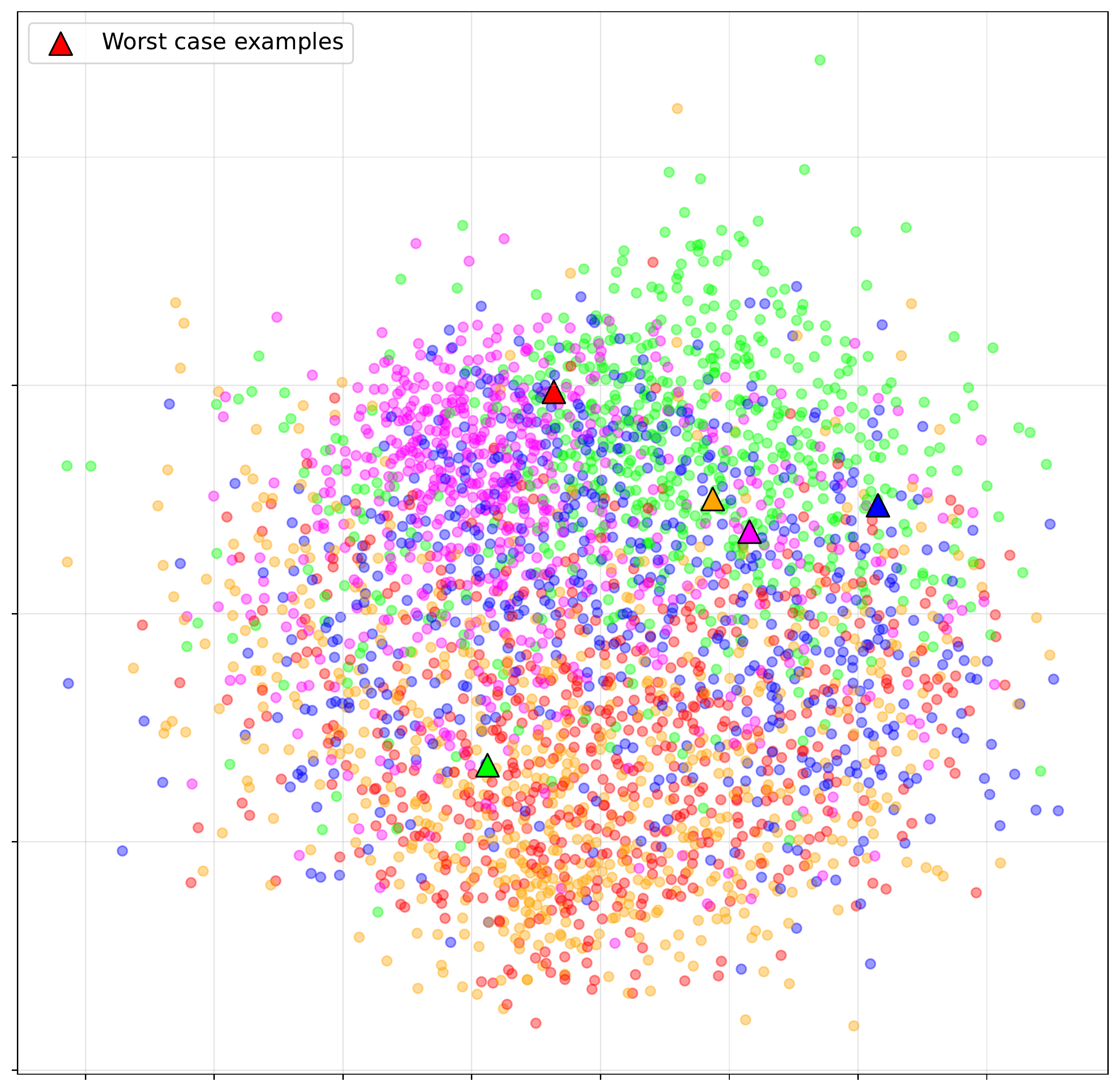}
                                            &           
                                \includegraphics[width=0.2\textwidth]{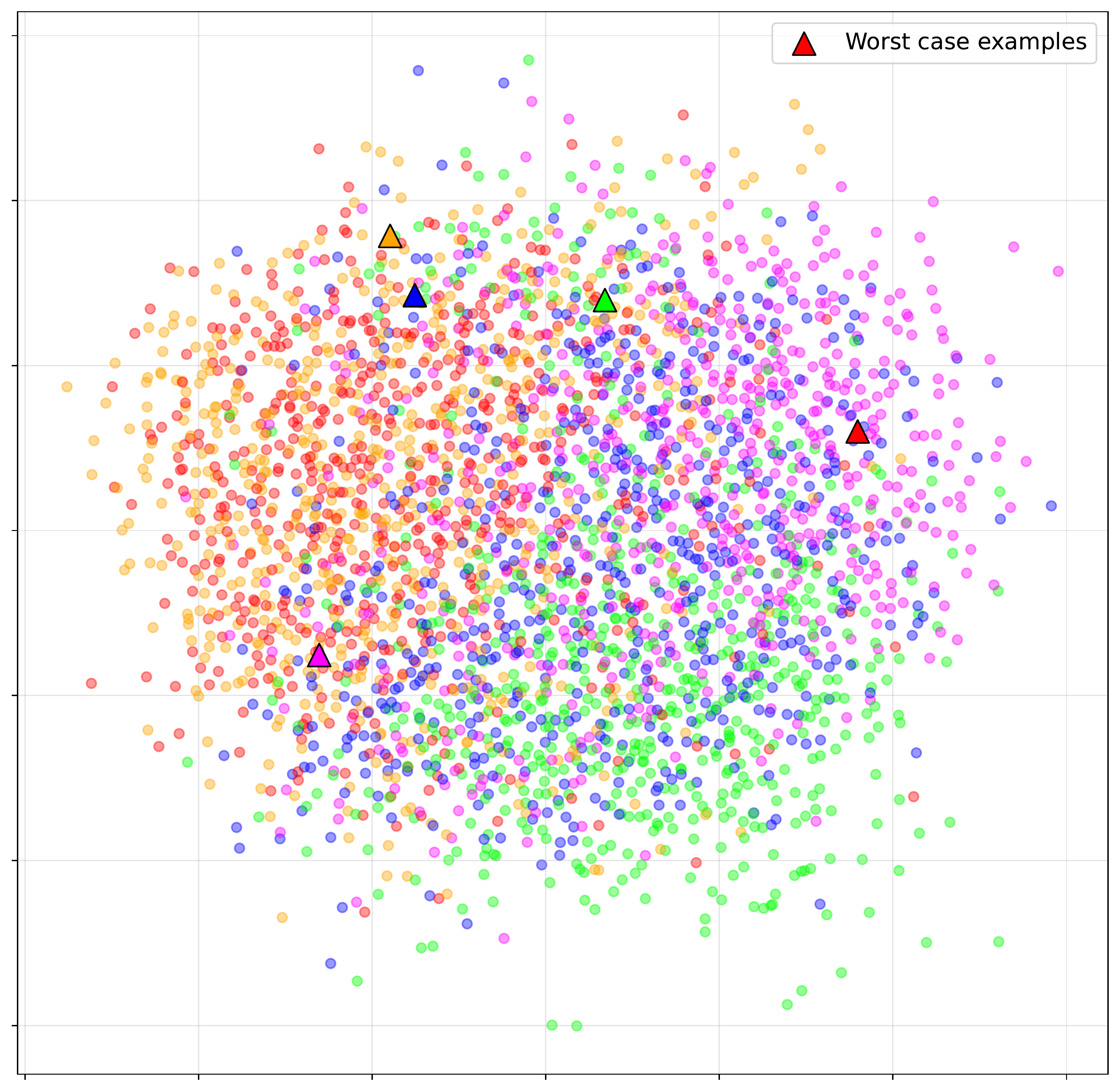}
                                \\
    \bottomrule
    \end{tabular}

\end{table}


\end{document}